\documentclass{article}

\usepackage{microtype}
\usepackage{graphicx}
\usepackage{booktabs} % for professional tables
\usepackage{physics}
\usepackage{adjustbox}
\usepackage{tabularx}
\usepackage{xspace}
\usepackage{multicol}
\usepackage{multirow}
\usepackage{bbm}
\usepackage{tikz}
\usepackage{subfig}
\usepackage{setspace}
\usetikzlibrary{automata, arrows.meta, positioning,shapes}

\usepackage{xcolor}
\definecolor{sb_blue}{RGB}{31,119,180}
\definecolor{sb_orange}{RGB}{255,127,14}
\definecolor{sb_green}{RGB}{44,160,44}
\definecolor{sb_red}{RGB}{214,39,40}
\definecolor{sb_purple}{RGB}{148,103,189}
\definecolor{sb_brown}{RGB}{140,86,75}
\definecolor{sb_pink}{RGB}{227,119,194}

\definecolor{sb_gray}{RGB}{234, 234, 242}

\definecolor{sb_blue_025}{RGB}{199, 221, 236}
\definecolor{sb_orange_025}{RGB}{255, 223, 195}

\definecolor{sb_orange_02}{RGB}{255, 229, 207}
\definecolor{sb_blue_015}{RGB}{222, 235, 244}

\definecolor{sublimegray}{RGB}{116,112,93}

\definecolor{e_1}{RGB}{59,82,139}
\definecolor{e_2}{RGB}{33,145,140}
\definecolor{e_3}{RGB}{94,201,98}

\definecolor{m_10}{RGB}{18,13,49,}
\definecolor{m_9}{RGB}{51,16,103,}
\definecolor{m_8}{RGB}{89,21,126,}
\definecolor{m_7}{RGB}{126,36,130,}
\definecolor{m_6}{RGB}{163,48,126,}
\definecolor{m_5}{RGB}{200,62,115,}
\definecolor{m_4}{RGB}{233,84,98,}
\definecolor{m_3}{RGB}{250,125,94,}
\definecolor{m_2}{RGB}{254,169,115,}
\definecolor{m_1}{RGB}{254,211,149,}

\definecolor{v_10}{RGB}{72,33,115,}
\definecolor{v_9}{RGB}{67,62,133,}
\definecolor{v_8}{RGB}{56,88,140,}
\definecolor{v_7}{RGB}{45,112,142,}
\definecolor{v_6}{RGB}{37,133,142,}
\definecolor{v_5}{RGB}{30,155,138,}
\definecolor{v_4}{RGB}{42,176,127,}
\definecolor{v_3}{RGB}{82,197,105,}
\definecolor{v_2}{RGB}{134,213,73,}
\definecolor{v_1}{RGB}{194,223,35,}

\definecolor{lhl}{RGB}{234, 245, 250}
\definecolor{hl}{RGB}{205, 232, 248}
\definecolor{lg}{RGB}{127, 127, 127}

\definecolor{ab_m0}{RGB}{4, 4, 20}
\definecolor{ab_m1}{RGB}{13, 10, 41}
\definecolor{ab_m2}{RGB}{25, 16, 63}
\definecolor{ab_m3}{RGB}{39, 18, 88}
\definecolor{ab_m4}{RGB}{56, 16, 108}
\definecolor{ab_m5}{RGB}{73, 16, 120}
\definecolor{ab_m6}{RGB}{87, 21, 126}
\definecolor{ab_m7}{RGB}{103, 27, 128}
\definecolor{ab_m8}{RGB}{118, 33, 129}
\definecolor{ab_m9}{RGB}{134, 39, 129}
\definecolor{ab_m10}{RGB}{150, 44, 128}
\definecolor{ab_m11}{RGB}{166, 49, 125}
\definecolor{ab_m12}{RGB}{183, 55, 121}
\definecolor{ab_m13}{RGB}{197, 60, 116}
\definecolor{ab_m14}{RGB}{213, 68, 109}
\definecolor{ab_m15}{RGB}{227, 78, 101}
\definecolor{ab_m16}{RGB}{238, 91, 94}
\definecolor{ab_m17}{RGB}{246, 108, 92}
\definecolor{ab_m18}{RGB}{250, 127, 94}
\definecolor{ab_m19}{RGB}{252, 144, 101}
\definecolor{ab_m20}{RGB}{254, 163, 111}

\def\ee{\mathrm{e}}

\usepackage{hyperref}
\newcommand{\po}{\phantom{0}}
\newcommand{\TF}{\texttt{TF}}
\newcommand{\SEQ}{\texttt{SEQ}}
\newcommand{\RL}{\texttt{RL}}
\newcommand{\NC}{{\mathsf{NC}^1}}
\newcommand{\TC}{{\mathsf{TC}^0}}
\newcommand{\AC}{{\mathsf{AC}^0}}

\newcommand{\porl}{Partially Observable RL\xspace}
\newcommand{\POrl}{Partially Observable RL\xspace}

\newcommand{\ssm}{state space modeling\xspace}

\usepackage{color}
\usepackage{soul}
\newcommand{\mhl}[1]{\sethlcolor{hl}\hl{#1}}
\newcommand{\mlhl}[1]{\sethlcolor{sb_orange_025}\hl{#1}}
\newcommand{\parity}{\texttt{PARITY}\xspace}
\newcommand{\sfive}{{\texttt{SYM(5)}}\xspace}
\newcommand{\mln}[2]{\mathcal M^{#1}\qty(#2)\xspace}
\newcommand{\ml}[1]{\mathcal M^{#1}\xspace}
\newcommand{\nccomplete}{{((0+1)^3(01^*0+1))^*}}
\newcommand{\RR}{{\mathbb{R}}}
\usepackage{enumitem}

\newcommand{\note}[1]{#1}

\newcommand{\cam}[1]{{#1}}
\usepackage[accepted]{icml2024}

\usepackage{amsmath}
\usepackage{amssymb}
\usepackage{mathtools}
\usepackage{amsthm}

\usepackage[capitalize,noabbrev]{cleveref}

\theoremstyle{plain}
\newtheorem{theorem}{Theorem}[section]
\newtheorem{proposition}[theorem]{Proposition}
\newtheorem{lemma}[theorem]{Lemma}
\newtheorem{corollary}[theorem]{Corollary}
\theoremstyle{definition}
\newtheorem{definition}[theorem]{Definition}

\theoremstyle{remark}
\newtheorem{remark}[theorem]{Remark}

\usepackage[textsize=tiny]{todonotes}

\begin{document}

\twocolumn[
\icmltitle{Rethinking Transformers in Solving POMDPs}



\icmlsetsymbol{equal}{*}

\begin{icmlauthorlist}
\icmlauthor{Chenhao Lu}{thu}
\icmlauthor{Ruizhe Shi}{thu,equal}
\icmlauthor{Yuyao Liu}{thu,equal}
\icmlauthor{Kaizhe Hu}{thu}
\icmlauthor{Simon S. Du}{uw}
\icmlauthor{Huazhe Xu}{thu,qizhi,shailab}
\end{icmlauthorlist}

\icmlaffiliation{thu}{IIIS, Tsinghua University}
\icmlaffiliation{qizhi}{Shanghai Qi Zhi Institute}
\icmlaffiliation{uw}{University of Washington}
\icmlaffiliation{shailab}{Shanghai AI Lab}
\icmlcorrespondingauthor{Chenhao Lu}{luch21@mails.tsinghua.edu.cn}
\icmlcorrespondingauthor{Huazhe Xu}{huazhe\_xu@mail.tsinghua.edu.cn}
\icmlkeywords{Transformer, RNN, POMDP, Linear RNN, Partially Observable RL}

\vskip 0.3in
]



\printAffiliationsAndNotice{\icmlEqualContribution} % otherwise use the standard text.

\begin{abstract}

Sequential decision-making algorithms such as reinforcement learning~(RL) in real-world scenarios inevitably face environments with partial observability. This paper scrutinizes the effectiveness of a popular architecture, namely Transformers, in Partially Observable Markov Decision Processes (POMDPs) and reveals its theoretical and empirical limitations. 
We establish that regular languages, which Transformers struggle to model, are reducible to POMDPs. This poses a significant challenge for Transformers in learning POMDP-specific inductive biases, due to their lack of inherent recurrence found in other models like RNNs.
This paper casts doubt on the prevalent belief in Transformers as sequence models for RL and proposes to introduce a point-wise recurrent structure. The Deep Linear Recurrent Unit (LRU) emerges as a well-suited alternative for \porl, with empirical results highlighting the sub-optimal performance of Transformer and considerable strength of LRU. \cam{Our code is open-sourced\footnote{\url{https://github.com/CTP314/TFPORL}}.}

\end{abstract}
\vspace{-0.3in}
\section{Introduction}
\label{sec:intro}
Reinforcement Learning (RL) in the real world confronts the challenge of incomplete information~\cite{challenge_in_RL} due to partial observability, necessitating decision-making based on historical data. The design of RL algorithms under partial observability, denoted as \textbf{\porl}~\cite{posd,prs,recurrentRL}, typically employs a hierarchical structure combining $(\SEQ,\RL)$. This structure involves firstly feeding the history into a sequence model $\SEQ$, such as Recurrent Neural Network (RNN)~\cite{rnn} or Long Short-Term Memory (LSTM)~\cite{lstm}, yielding a hidden state containing past information, then processing it using existing $\RL$ algorithms.

% GPT rewrite: Reinforcement Learning (RL), when applied in real-world scenarios, confronts the challenge of incomplete information, a common issue due to partial observability~\cite{challenge_in_RL}. This necessitates the enhancement of decision-making capabilities, leveraging historical data. Addressing this, the design of RL algorithms under conditions of partial observability, referred to as \textbf{\porl}\cite{posd,prs,recurrentRL}, typically adopts a hierarchical architecture combining $(\SEQ,\RL)$. Within this framework, historical data is initially processed by a sequence model $\SEQ$, such as a Recurrent Neural Network (RNN)\cite{rnn} or Long Short-Term Memory (LSTM)~\cite{lstm}, which generates a hidden state encapsulating past information. Subsequently, this state is utilized by existing $\RL$ algorithms for further processing.

Regarding the sequence model, Transformer \cite{Transformers}, renowned for its achievements in the natural language processing (NLP) domain~\cite{GPT2,GPT3,GPT4}, stands out as a prominent candidate. Transformers have shown a strong ability to handle contexts in a non-recurrent manner. Compared with their recurrent counterpart like RNNs, the advantages of Transformers as a sequence model shine in several aspects: 1) long-term memory capacity, as opposed to RNNs with rapid memory decay~\cite{tfshine, gtrxl}; 2) effective representation learning from context for specific tasks~\cite{tfefficientwm, ad, mdt, tfwmhappy}, benefiting meta-RL or certain environments~\cite{atari}; 3) stronger learning ability on large-scale datasets~\cite{vpt}. 

% GPT rewrite: Regarding the choice of the sequence model, the Transformer~\cite{Transformers}, renowned for its significant achievements in the natural language processing (NLP) domain~\cite{GPT2,GPT3,GPT4}, stands out as a formidable option. Unlike their recurrent counterparts such as RNNs, Transformers excel in processing contexts through a non-recurrent approach. The superiority of Transformers over RNNs is evident in several key areas: 1) They possess a robust long-term memory capacity, in contrast to the rapid memory decay characteristic of RNNs~\cite{tfshine, gtrxl}; 2) They demonstrate an enhanced capability for context-specific representation learning, which is particularly advantageous for meta-RL or certain environments~\cite{tfefficientwm, ad, mdt, tfwmhappy, atari}; 3) Their proficiency in learning from large-scale datasets is markedly stronger~\cite{vpt}.

% In recent years, there have been many applications of Transformers in RL, such as Decision Transformer (DT)~\cite{dt} and its variants~\cite{qdt,elasticdt} in offline RL; GTrXL~\cite{gtrxl} and online DT~\cite{onlinedt} in online RL; and the Transformer State Space Model~\cite{TSSM} as world models.

However, deploying Transformers in \porl introduces challenges, commonly manifesting as \cam{sample inefficiency}~\cite{gtrxl, tfshine}. This issue is similarly observed in computer vision (CV) and is attributed to the data amount that Transformers require to learn the problem-specific inductive biases~\cite{vit}. 
% Furthermore, the inductive biases inherent in RNNs and POMDPs exhibit similarities, as both are designed to model Hidden Markov Models (HMMs), while Transformers do not. 
While it is validated in CV, it remains unknown whether data amount is the key ingredient in decision making. Hence, a natural question arises:
\textit{Can Transformers effectively solve decision-making problems in POMDPs with sufficient data?}
%As Partially Observable Markov Decision Process (POMDP) is the standard model that captures \porl, we naturally ask,

% GPT rewrite: The integration of Transformers into \porl, however, presents distinct challenges, notably training instability and sample inefficiency~\cite{gtrxl, tfshine}. This issue is not unique to \porl and is similarly observed in computer vision (CV), where it is attributed to the substantial data requirement of Transformers for learning problem-specific inductive biases~\cite{vit}. While this has been established in the CV domain, it's unclear if the volume of data is the crucial factor in decision-making contexts. Consequently, this leads to an intriguing inquiry:
%\textit{Can Transformers effectively address decision-making challenges in POMDPs, provided they have access to sufficient data?}

In this work, we investigate this critical question and challenge the conventional wisdom. We argue that Transformers cannot solve POMDP even with massive data. %Our inspiration comes from the following
This stance is inspired by a key observation: 
%/ other wording here NLP and CV, they lag in modeling formal languages. 
While most RNNs are complete for regular languages, Transformers falter to model them~\cite{chomsky, hahn} % including \parity, 
. A notable example is their struggle with tasks like \parity, which is to determine the parity of the occurrence of ``1'' in a binary string.
%Given the disparities between natural and formal languages~\cite{critical}, this limitation has a relatively minor impact on Transformer's performance in NLP domain. 
We hypothesize that, in POMDPs, this limitation becomes pronounced due to the close relationship between regular languages and POMDPs.

% GPT rewrite: In our study, we delve into this critical question, challenging the prevailing assumptions. We contend that even with extensive data, Transformers are inadequate for solving POMDPs. This stance is inspired by a key observation: while most RNNs are adept at modeling regular languages, Transformers falter in this aspect~\cite{chomsky, hahn}. A notable example is their struggle with tasks like \parity, which involves determining the parity of "1" occurrences in a binary string. We hypothesize that this limitation is accentuated in POMDPs, given the close correlation between regular languages and the dynamics of POMDPs.

To elaborate further, regular languages exhibit a direct correspondence with Hidden Markov Models (HMMs)~\cite{hmmsrl-evid1}, and POMDPs can be regarded as HMMs augmented with an incorporated decision-making process. We further establish that regular languages can be effectively reduced to POMDPs. \note{
From the computational complexity perspective, the parallel structure of the Transformer makes it equivalent to a constant-depth computation circuit. Some regular languages fall outside of this complexity class, making the POMDP problems derived from them harder, and Transformer would struggle to solve them. This is demonstrated both theoretically and empirically in this study.
}

To \cam{alleviate} the limitations of Transformers \note{caused by the parallel structure}, we propose to introduce a pointwise recurrent structure. Upon reviewing current variants of sequence models with such a structure, we find that they can be broadly generalized as \note{linear RNNs.}
% a specific type of existing linear RNN model, \todo{} known as Deep Linear Recurrent Unit (LRU)~\cite{linearrnn}
\note{Based on extensive experiments over a range of sequence models over POMDP tasks with diverse requirements, we highlight LRU~\cite{linearrnn} as a linear rnn model well-suited for \porl.} %\note{We evaluate different sequence models on a range of POMDP tasks with diverse requirements, thereby validating our ideas.}
Our contributions are three-fold:

% GPT rewrite: Addressing the identified limitations of Transformers, we propose the integration of a pointwise recurrent structure. Our examination of current sequence models featuring this structure reveals that they can be broadly generalized under a specific category of linear RNN models, referred to as the Deep Linear Recurrent Unit (LRU)~\cite{linearrnn}. \note{Our extensive experimental analysis, encompassing a variety of sequence models, POMDP tasks, and diverse criteria, underscores the efficacy of LRU as a particularly suitable sequence model for \porl.}

\begin{itemize}[leftmargin=2em,itemsep=0em,topsep=0em]
\item We demonstrate the theoretical limitations of Transformers as sequence model backbones for solving POMDPs, through rigorous analysis.
\item To better utilize the inductive bias of the sequence model, we study the advantages of the Transformer and the RNN, and advocate the linear RNN as a better-suited choice for solving POMDPs\note{, taking advantage of both models.}
% , combining both of them.
\item Through extensive experiments across various tasks, We compare the capabilities exhibited by various sequence models across multiple dimensions. Specifically, we show that Transformers exhibit sub-optimal performance as the sequence model in \cam{certain POMDPs}, while highlighting \cam{the strength} of linear RNNs when assessed comprehensively.
\end{itemize}

% GPT rewrite: 
% \begin{itemize}[leftmargin=2em,itemsep=0em,topsep=0em]
% \item We provide a detailed theoretical analysis demonstrating the limitations of Transformers as sequence model backbones for POMDPs.
% \item In pursuit of an optimal inductive bias in sequence models, we analyze the strengths of both Transformers and RNNs. Consequently, we propose the LRU model, \kz{which synergizes the advantages of both}, as a more effective solution for POMDPs.
% \item Our extensive experimental framework spans various tasks, allowing us to comparatively assess the performance of different sequence models. These experiments substantiate our claims, showing Transformers' suboptimal performance in \porl and highlighting the comprehensive strengths of the LRU model.
% \end{itemize}

\section{Related Work}
\textbf{Theoretical limitations of Transformers}.
There is a substantial body of work investigating the theoretical limitations of Transformers from the perspective of computational complexity and formal language. For example, \citet{chomsky,rsa} experimentally verifies that 
% the formal language level corresponding to Transformers is lower than that of regular languages, to which RNN-like models belong. 
\note{RNNs can recognize regular languages, but Transformers are unable to achieve this.}
Additionally, \citet{attentionlim} demonstrates that Transformers are not robust in handling sequence length extrapolation. Moreover, \citet{parallelism,attn-circuits} point out that, under limited precision, $\TC$ serves as an upper bound for the computational power of Transformers. Applying this result, \citet{cot} illustrates the challenges Transformers face in solving practical problems such as arithmetic operations and linear systems of equations. \cam{Currently, works such as \citet{tfshine, popgym,s4wm} discuss the pros and cons of transformers in RL algorithms, with a focus on analyzing the advantages or providing simple evaluations. In contrast, integrating relevant theories from formal languages, we offer a new theoretical perspective on analyzing the limitations of transformers in RL.}
% Inspired by these works, this paper contributes by proving that $(\TF, \RL)$ universally solves POMDP by reducing regular language recognition to POMDP problems. Additionally, we experimentally validate our theoretical results on some simple and insightful tasks.

\textbf{Variants of sequence models for handling long contexts}.
Multiple variants of mainstream sequence models designed to handle long contexts have provided significant inspiration for this work. In the case of RNN-like models, addressing the issue of rapid memory decay has led to the emergence of linear RNNs~\cite{s4,linearrnn}, which remove activation functions in the recurrence part. These models have demonstrated excellent performances in benchmarks for long-range modeling~\cite{lra}. For Transformers, to tackle limited training length and inefficient inference, current studies emphasize the introduction of recurrence. Recurrence in Transformers can be categorized into two types: 1) \textit{chunkwise recurrence}, which processes parts exceeding the context length using a recurrent block with minimal alterations to the original parallel structure~\cite{trxl,blockrecurrent}; 2) \textit{pointwise recurrence}, which derives recurrence representations by linearizing attention~\cite{retnet,rwkv,fwp,fart}. % This paper argues that linear RNNs and pointwise recurrence Transformers share a common goal and are suitable for solving POMDP problems.
\note{We argue that linear RNNs and pointwise recurrence Transformers ultimately converge to a similar solution, incorporating the strength of both approaches and are suitable for solving \porl problems}. 

\textbf{Applications of sequence models in RL}. In recent years, there have been many applications of Transformers in RL, such as Decision Transformer (DT)~\cite{dt}, its variants~\cite{qdt,elasticdt} in offline RL; GTrXL~\cite{gtrxl}, Online DT~\cite{onlinedt} in online RL, and the Transformer State Space Model~\cite{TSSM} as world models. There are works~\cite{wikirl,lamo} showing that the inductive bias of pre-trained Transformers could help RL, where the states are fully observable. \citet{defog} uses Transformer to solve a specific partially observable setting, frame drops, whereas demanding additional assumptions on the prior distribution. On the other hand, recent works comparing Transformer-based and RNN-based approaches in \porl empirically support our idea that Transformers have weaknesses in partially observable environments~\cite{popgym, s4wm}. In many cases, simpler architectures like RNNs and LSTMs prove to be more effective. For instance, DreamerV3~\cite{dreamerv3}, which adopts GRU~\cite{gru} as the backbone of the Recurrent State Space Model, has outperformed previous Transformer-based approaches like VPT~\cite{vpt} and IRIS~\cite{tfefficientwm}. \cam{Additionally, there has been a recent line of research on the application of linear RNNs in RL~\cite{irie2021going,s4rl,samsami2024mastering,}, which has shown promising results. This paper will incorporate insights into the limitations of Transformers to analyze why this would be a natural choice.}

% Across numerous scenarios, it becomes evident that recurrent models often demonstrate superior performance in partially observable environments compared with Transformers~\cite{popgym, s4wm}. 

\section{Preliminaries}
\label{sec:pre}
\textbf{Sequential neural network.} Sequential Neural Networks are a type of deep learning model for sequence modeling. Given a input sequence $\qty{u_i}_{i=1}^n$, the model learns a hidden state sequence $\qty{x_i}_{i=1}^n$ and yields the output sequence $\qty{y_i}_{i=1}^n$. There are currently two mainstream methods for computing the hidden state $x_i$: 
\begin{itemize}[leftmargin=2em,itemsep=0em,topsep=0em]
    \item Recurrent-like: $x_t=\sigma(Ax_{t-1}+Bu_{t-1}+c)$ where $\sigma$ is the activation function; 
    \item Attention-like: $x_t=\operatorname{attn}\qty(W^Q\tilde{U},W^K\tilde{U},W^V\Tilde{U})_t$. Here $\tilde{U}_i=u_i+p_i$, $p_i$ is a position embedding, and $\operatorname{attn}(Q,K,V)$ is defined as $\operatorname{softmax}(QK^\top+M)V$, where $M$ is an attention mask.
\end{itemize}

\begin{figure}[ht]
    \vskip -0.1in
    \centering
    \subfloat[(Pointwise) Recurrent\label{fig:recurrent}]{
        \includegraphics[width=0.465\columnwidth]{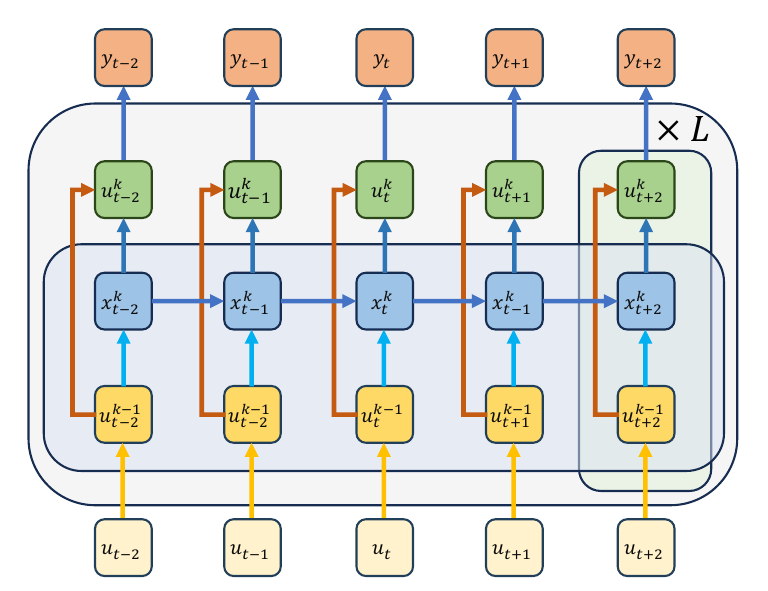}
    }
    \hfill
    \subfloat[(Parallel) Attention\label{fig:attention}]{
        \includegraphics[width=0.465\columnwidth]{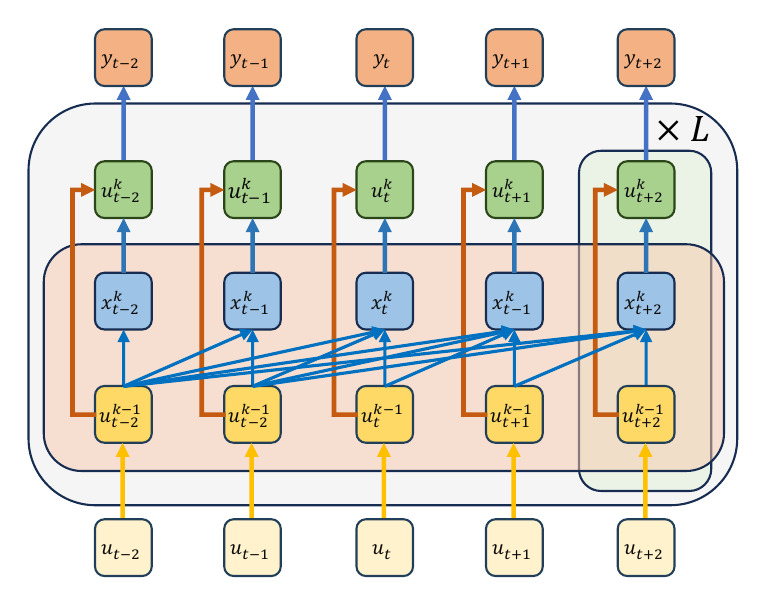}
    }
    \caption{The Left figure indicates the general structure of recurrent-like sequential neural networks and the right figure represents the attention-like ones. \cam{Both of them can be deepened by pointwise transformations and skip connections.}}
    \label{fig:seq-model}
    \vskip -0.2in
\end{figure}

\cam{Current sequence models use multiple layers.} The overall structure is illustrated in Figure~\ref{fig:seq-model}. For recurrent models like RNNs, GRUs, etc., \cam{
the output of the previous layer is directly used as the input for the next layer.}, while for Transformers, \cam{the pointwise transformations typically take the form of \cam{MLPs} with skip connections.}

\textbf{POMDP.} A Partially Observable Markov Decision Process (POMDP) $\mathcal M$ can be defined as $(S, A, T, R, \Omega, O, \gamma)$ \cite{ASTROM1965174}. 
% The observation history $h_t$ is defined as $\qty{\qty(o_i, a_i, r_i)}_{i=1}^t$, and the trajectory history is $\qty{(s_i, o_i, a_i, r_i)}_{i=1}^t$, where $s_i, o_i, a_i, r_i$ indicate that at time $i$, the agent is in state $s_i\in S$, observes $o_i\in \Omega \sim O(\cdot|s_i)$, takes action $a_i\in A$, receives reward $R(s_i, a_i)$ and would transit to $s_{i+1}\sim T(\cdot|s_i,a_i)$. $\gamma\in [0,1)$ is the discount factor. The agent's policy based on the observation history is denoted as $\pi(\cdot|h_{t-1}, o_t)$. 
At time $i$, the agent is in state $s_i\in S$, observes $o_i\in \Omega \sim O(\cdot|s_i)$, takes action $a_i\in A$, receives reward $R(s_i, a_i)$ and would transit to $s_{i+1}\sim T(\cdot|s_i,a_i)$. The agent's policy based on the observation history $h_t=\qty{\qty(o_i, a_i, r_i)}_{i=1}^t$ is denoted as $\pi(\cdot|h_{t-1}, o_t)$.
We say an algorithm $\mathcal A$ that can \textit{solve} POMDP as being able to find the optimal policy $\pi^\star$ for given POMDP $\mathcal M$ where: \[
    \pi^\star=\operatorname*{argmin}_\pi\mathop{\mathbb{E}}_{\substack{a_t\sim\pi(\cdot|h_{t-1},o_t)  \\ s_t\sim T(\cdot|s_{t-1},a_t)\\ o_t\sim O(\cdot|s_t)}}\qty[\sum_{t=0}^\infty\gamma^t R(s_t,a_t)]\;.
\]
% \kz{which requires A to recover the real state s from observation o}

\textbf{DFA \& regular language.} Deterministic Finite Automata (DFA) can be defined as \cam{$A=(S,\Sigma,T,s_0,F)$}, where \cam{$S$} is a finite set of states, $\Sigma$ is a finite set of symbols, \cam{$T:S\times \Sigma\to S$} is the transition function, \cam{$s_0$} is the start state, and $F\subseteq Q$ is the set of accepting states. A string $w$ whose \cam{$i$-th symbol is $w_i$} is \textit{accepted} by $A$ if \cam{$\exists (s_0,s_1,\ldots,s_n)$}, s.t. \cam{$s_i\in S$} for $1\le i\le n$, \cam{$s_i=T(s_{i-1},w_i)$} and \cam{$s_n\in F$}. $L$ is a regular language if it is \textit{recognized} by DFA $A$, that is, $L=\{w:A\text{ \textit{accepts} }w\}$.
% Regular languages are defined as $\{L|L\text{ is \textit{recognized} by a DFA }A\}$.

% \textbf{Observability \& Constructibility.} Given any input $u$, if the state $x(t_0)$ is uniquely determined by output $y(t),t\in[t_0,t_f]$, then the initial state $x(t_0)$ is \textit{observable}. If the state $x(t_f)$ is uniquely determined by output $y(t),t\in[t_0,t_f]$, then the initial state $x(t_0)$ is \textit{constructible}.

\section{Limitations of Transformer in \POrl}
\label{sec:lim-of-tf-in-porl}
% \note{We begin our investigation on whether combining Transformer and reinforcement learning$(\TF,\RL)$ can solve POMDPs. There is a well-known correspondence between the Hidden Markov Model (HMM) and Stochastic Regular Language \cite{hmmsrl-evid1}. 
% A POMDP problem can be viewed as an HMM with a decision-making process by definition.}
% \kz{connection}
\note{
RL algorithms typically take as inputs the current state
with the assumption of the Markov property. In partially observable environments that lack the Markov property, the hidden state extracted by $\SEQ$ from the observation history is thus anticipated to contain the information of the \textit{real state} to benefit subsequent $\RL$. 
% Therefore, addressing POMDPs requires the $\SEQ$ component to have an inductive bias that reconstructs the real state.
}

\note{Notably, \citet{chomsky} points out that Transformers (referred to as $\TF$s) fail to recognize regular languages. Inspired by the significant correspondence between regular languages and POMDPs (details in Definition~\ref{def:pomdp-regular}), we naturally conjecture that $\TF$ is not capable of retrieving the information of \textit{real state} from partial observations accurately, which would lead to a decline in the performance of the pipeline $(\SEQ,\RL)$.
% Comparing the definitions of the two, their states can correspond to each other, and characters can be mapped to observations and actions. When a DFA enters state $q$, regardless of the previous inputs, its final acceptance depends only on the subsequent inputs, the same as the Markovian property.
} 

Building upon this view, this section first shows that solving a POMDP problem is harder than solving a regular language problem. Afterward, we will introduce two theoretical results to elucidate the limitations of Transformers, supported by simple examples for illustrations. Consequently, it is inferred that $(\TF,\RL)$ cannot address POMDPs generally.

\subsection{Reduction from Regular Language to POMDP}

\begin{proposition} 
\label{prop:r-to-pomdp}
If an algorithm $\mathcal A=(\texttt{SEQ},\texttt{RL})$ can solve POMDPs, then given a regular language $L$, $\mathcal A$ can recognize $L$ by solving a POMDP problem $\mathcal M$.
% A regular language $L$ can reduce to a POMDP $\ml{L}$ so that an algorithm $\mathcal A$ can recognize $L$ by solving $\ml{L}$.
\end{proposition}
\textbf{Proof idea.} We construct a POMDP $\mathcal M$, such that each state represents a transition in $L$, and the observation at timestep $t$ is the corresponding character $w_t$. The agent could output \textit{accept} or \textit{reject}, and the reward is assigned if and only if the final output aligns with the acceptance of the string $w=w_0\ldots w_{t-1}$ in $L$. In this way, the optimal policy $\pi$ on $\mathcal M$ is to \textit{accept} all $w\in L$ and \textit{reject} all $w\notin L$, so if an algorithm $\mathcal A$ can solve POMDP $\mathcal M$, then $\mathcal A$ can recognize $L$. Proof details are deferred to Appendix~\ref{subsec:proof-r-to-pomdp}. 

\begin{definition}[\note{POMDP derived from regular language $L$}] 
\label{def:pomdp-regular}
Given a regular language $L$, the POMDP derived as Proposition~\ref{prop:r-to-pomdp} is denoted as $\ml{L}$. For an integer $n$, $\mln{L}{n}$ represents
% instances \kz{what instance} where the input length does not exceed $n$.
\note{a special case of $\ml{L}$ whose horizon is no longer than $n$.} 
\end{definition}

\begin{remark}
% When implementing \porl, the problems typically involve tackling a finite horizon length $n$, where $\mln{L}{n}$ represents a scenario characterized by such a finite horizon.
\note{When implementing RL algorithms, especially in online settings, it is common to set a truncated time $n$ for training and evaluation purposes. For $\mln{L}{n}$ with maximum horizon $n$, since observation during training and evaluation come from the same distribution, we can analogize it to \textit{fitting} in supervised learning. Furthermore, in partially observable cases, historical information needs to be considered. If there is no time limit set, there is a need for \textit{length extrapolation}, which we can analogize to \textit{generalization} in supervised learning, as is captured by $\ml{L}$.}
\end{remark}

In Figure~\ref{fig:reduction-parity}, we illustrate how to construct $\mathcal M^L$. 
% More examples can be found at (cf. Appendix~\ref{exampls:r-to-pomdp}). 
% reduction
We also provide experiments to verify the \note{reduction} (cf. Section~\ref{exp:r}).

% \begin{figure*}[htbp]
%     \vskip 0.2in
%     \centering
%     \subfloat[]{
%         \begin{adjustbox}{width=0.38\textwidth}
%         \input{DFA/dfa_parity}
%         \end{adjustbox}
%     }
%     \hfill
%     \subfloat[]{
%         \begin{adjustbox}{width=0.52\textwidth}
%         \input{DFA/pomdp_parity}
%         \end{adjustbox}
%     }
%     \caption{Left: Illustration for DFA of \parity. Right: Illustration for $\ml{\parity}$ \td{Add more description}}
%     \label{fig:reduction-parity}
%     \vskip -0.2in
% \end{figure*}

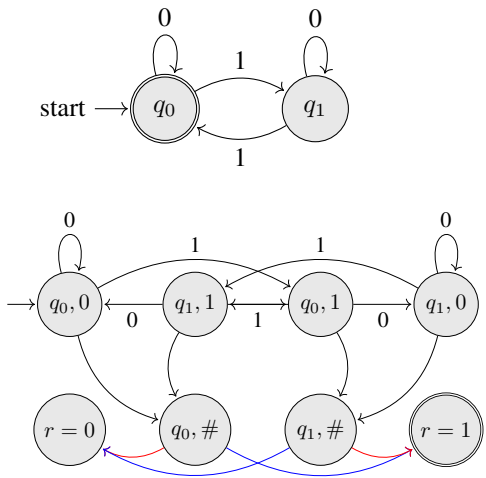
\begin{figure}[htbp]
    \vskip -0.1in
    \centering
    \subfloat{
        \begin{adjustbox}{}
        \begin{tikzpicture}[shorten >=1pt,node distance=2cm,on grid,auto]
  \tikzstyle{every state}=[fill={rgb:black,1;white,10}]

    \node[state,initial,accepting]   (q_0)                    {$q_0$};
    \node[state] (q_1)  [right of=q_0]    {$q_1$};

    \path[->]
    (q_0) edge [loop above] node {0}    (   )
          edge [bend left]  node {1}    (q_1)
    (q_1) edge [bend left]  node {1}    (q_0)
          edge [loop above] node {0}    (   );
\end{tikzpicture}
        \end{adjustbox}
    }
    \newline
    \subfloat{
        \begin{adjustbox}{width=0.8\columnwidth}
        \begin{tikzpicture}[shorten >=1pt,node distance=2cm,on grid,auto]
  \tikzstyle{every state}=[fill={rgb:black,1;white,10}]

    \node[state,initial,initial text={}]   (q0q0)       {$q_0,0$};
    \node[state] (q1q0)  [right of=q0q0]    {$q_1,1$};
    \node[state] (q0q1)  [right of=q1q0]    {$q_0,1$};
    \node[state] (q1q1)  [right of=q0q1]    {$q_1,0$};
    \node[state] (q0E)  [below of=q1q0]    {$q_0,\#$};
    \node[state] (q1E)  [below of=q0q1]    {$q_1,\#$};
    \node[state] (F) [left of=q0E] {$r=0$};
    \node[state, accepting] (S) [right of=q1E] {$r=1$};

    \path[->]
    (q0q0) edge [bend left]  node {1}    (q0q1)
           edge [loop above] node {0}    ()
           edge [bend right] node {} (q0E)
    (q1q0) edge node {}   (q0q1)
           edge node {0}   (q0q0)
           edge [bend right] node {} (q0E)
    (q0q1) edge node {1}   (q1q0)
           edge node[below] {0}   (q1q1)
           edge [bend left] node {} (q1E)
    (q1q1) edge [bend right] node[above] {1}   (q1q0)
           edge [loop above] node {0}    ()
           edge [bend left] node {} (q1E)
    (q0E)  edge [blue] [bend right] node {} (S)
           edge [red] [bend left] node {} (F)
    (q1E)  edge [red] [bend right] node {} (S)
           edge [blue] [bend left] node {} (F)
    % (q0q0) edge [bend left]  node {1}    (q0q1)
    %        edge [loop above]  node {0}
    % (q0q0) edge [bend left]  node {1}    (q0q1)
    %        edge [loop above]  node {0}
    ;
\end{tikzpicture}
        \end{adjustbox}
    }
    \caption{Above: Illustration for DFA of \parity. There are two states $q_0$ and $q_1$, where $q_0$ is both the initial state and the accepting state. The transitions are plotted in gray arrows. Below: Illustration for $\ml{\parity}$. The states are $(q_i,w)$ where $i\in\{0,1\}$ and $w\in\{0,1,\#\}$, and the agent could observe $w$. The initial state are randomly sampled from $(q_0,w)$. The stochastic transitions are plotted in gray arrows. At final state $(q_i,\#)$, \textcolor{blue}{blue} arrows stand for choosing \textit{accept}, and \textcolor{red}{red} arrows stand for choosing \textit{reject}.}
    \label{fig:reduction-parity}
    \vskip -0.2in
\end{figure}

\subsection{Limitations in Fitting: Solving $\mln{L}{n}$}
\label{subsec:fitting}
While prior research has claimed universality for Transformers, specifically proving their Turing completeness and ability to approximate any seq-to-seq function on compact support \cite{tf-turing-1, tf-turing-2, tf-no-expected, tf-universal}, it is crucial to note certain impractical assumptions underlying these assertions as they often rely on assumptions of infinite precision and finite length~\cite{approx-survey}.
% and their proofs involve intricate parameter constructions, introducing biases that may not align with real-world scenarios \cite{approx-survey}.

In this subsection, we assume that $(\TF, \RL)$ (\TF~denotes a Transformer) is a \textbf{log-precision model} that all values in the model have $O(\log n)$ precision, where $n$ is the input length. This assumption aligns with reality since computer floating-point precision is typically 16, 32, or 64 bits, smaller than the sequence lengths commonly handled by sequence models. Based on this assumption, 
% we can infer the limitations of the Transformer in fitting: there exists a class of POMDPs, for which achieving solutions with $(\TF,\RL)$ would demand parameter quantities that are practically unattainable.
\note{there exists a class of POMDPs, for which achieving solutions with $(\TF,\RL)$ would demand an excessively large quantity of parameters.} \cam{This type of problem can be directly mapped to a type of circuit complexity, with its definition provided in the appendix \ref{sec:circuit}.}

\begin{theorem}
\label{thm:lim-in-tf-fit}
Assume $\TC\neq\NC$. Given an $\NC$ complete regular language $L$, for any depth $D$ and a \cam{any polynomial $\operatorname{poly}(n)$}, there exists a length $n$ such that no log-precision $(\TF,\RL)$ with depth $D$ and hidden dimension \cam{$d\le \operatorname{poly}(n)$} can solve $\mln{L}{n}$.
\end{theorem}
\textbf{Proof idea.} At the heart of the proof is a contradiction achieved through circuit complexity theory~\cite{circuitcomplexity}. \citet{parallelism} has shown that $\TC$ circuits can simulate a log-precision Transformer with constant depth and polynomial hidden dimensions. Consequently, if $(\TF, \RL)$ can solve $\NC$ complete problems, it would cause both $\TC$ and $\NC$ complexities to collapse, a scenario generally deemed impossible~\cite{collapse}. Proof details of Theorem~\ref{thm:lim-in-tf-fit} are deferred to Appendix~\ref{subsec:proof-lim-in-tf-fit}. 

Following syntactic monoid theory \cite{semigroups} and Barrington's theorem \cite{barrington}, a significant number of regular languages are $\NC$ complete, such as the regular language $\nccomplete$ (cf. Appendix \ref{example:nc-complete}).

On the other hand, these two works inform us of another fact: for a regular language $L$, there are only two possibilities—either $L\in\NC$ complete or $L \in \TC$ (more specifically, $L\in\AC$). As of now, the question of whether problems solvable by log-precision Transformers belong to $\TC$ remains an open problem~\cite{parallelism}. However, numerous experimental results~\cite{chomsky,rsa} suggest that Transformers do not perform well in handling certain regular languages within $\TC$, such as $\parity$. In the next section, we demonstrate, from a generalization perspective, that Transformers cannot solve $\ml{L}$ for a broader range of regular languages $L$. 

\subsection{Limitations in Generalization: Solving $\ml{L}$}

% While truncated environments are usually established in simulative settings
When deploying the \porl algorithm, we anticipate it to demonstrate the ability of length generalization. 
% \kz{not just real-world..} 
In this subsection, we examine scenarios corresponding to $\mathcal{M}^L$ and no longer assume that the Transformer model operates with logarithmic precision.

Recent works~\cite{alibi,chomsky,randp} have empirically demonstrated that length extrapolation is a weakness of Transformers. Lemma~\ref{lemma:lim-in-tf-gen} theoretically indicates that for any Transformer with the dot-product softmax attention mechanism, robust generalization is not achievable as the input length increases.

\begin{lemma}[Lemma 5 in \citet{attentionlim}]
\label{lemma:lim-in-tf-gen}
%Let a softmax attention Transformer be given, 
Given a Transformer with softmax attention, let $n$ be the input length. If we change one input $u_i$ ($i < n$) to $u_i'$, then the change in the resulting hidden $x_n$ at the output layer is bounded by $O(D/n), D=\norm{u_i-u_i'}$ with constants depending on the parameter matrices.
\end{lemma}

\begin{theorem}
\label{thm:lim-in-tfrl-gen} 
% \note{Given an regular language $L$, let \[
%     S_n^k=\qty{(xa,ya)\in L_n\times\bar{L}_n:d(x,y)\le k,a\in\Sigma}\,,
% \]
% where $L_n$ denotes $\qty{x\in L:\abs{x}=n}$ and $d(\cdot,\cdot)$ denotes the number of positions at which the symbols are different. If there exists $k$ such that $\qty{n:\abs{S_n^k}\neq 0}$ is finite and $\RL$ is a Lipschitz function, then $(\TF,\RL)$ can not solve $\ml{L}$.}
Given an regular language $L$, let $c(n,a)=\#\qty{xa\in L:\abs{x}=n}$. If there exists $a\in\Sigma$ such that $\qty{n:0<c(n,a)<\abs{\Sigma}^n}$ are infinite, and $\RL$ is a Lipschitz function, then $(\TF,\RL)$ cannot solve $\mathcal M^L$.
\end{theorem}

\textbf{Proof idea}. Since $\Sigma$ is a finite set, $D$ is deterministic. Then we will prove for $L$ satisfying the conditions, there exists infinite $u,u'$ such that $u$ and $u'$ differ by only $1$ positions but $u\in L,u'\notin L$. According to Lemma \ref{lemma:lim-in-tf-gen}, the hidden states $x$ and $x'$ output by the Transformer differ by $O(1/n)$. Since $\RL$ is a Lipschitz function, the results of $\RL(x)$ and $\RL(x')$ also differ by $O(1/n)$. As $n$ increases, \note{information from non-current time steps will only have a negligible impact on the output of \RL.} 
% $\RL$ {tends towards unigram \question{true?} behavior and fails to solve} $\ml{L}$. 
Proof details of Theorem~\ref{thm:lim-in-tfrl-gen} are deferred to Appendix~\ref{subsec:lim-in-tfrl-gen}.

Observing that $\parity$ satisfies the conditions outlined in Theorem \ref{thm:lim-in-tfrl-gen}, we can derive Corrollary \ref{cor:tf-no-parity}. 
\vspace{0.1in}
\begin{corollary}
\label{cor:tf-no-parity}
If $L=\parity$ and $\RL$ is a Lipschitz function, then $(\TF,\RL)$ can not solve $\ml{L}$.
\end{corollary}
\vspace{-0.1in}
The Lipschitz property is commonly observed in widely used learning-based RL algorithms, such as employing MLPs to predict $Q$-values, $V$-values, or the probability distribution of the next action. For cases that do not satisfy the Lipschitz property, such as those relying on the maximum value rather than logits, \citet{overcoming} provides a constructive method for a Transformer that can recognize \note{\parity}. This scenario corresponds to the greedy policy based on $Q$-values. However, this theorem indicates that Transformers do not model sequences in a way that accurately reconstructs the \textit{real states}, \note{which makes it hard for $(\TF,\RL)$ to perform length extrapolation.}

% which holds significant implications for subsequent RL algorithms. \kz{the relation to generalization?}

% \td{add a summary}

% \subsection{Insights}
\subsection{\note{From POMDPs to Regular Languages}}
\label{sec:pomdp-to-regular}
Through illustrating the limitations of $(\TF, \RL)$ in handling POMDPs derived from regular languages, we demonstrate that there exist POMDP problems for which Transformers cannot effectively learn the corresponding inductive biases. 

\cam{This class of POMDP problems corresponding to regular languages can be divided into three levels based on circuit complexity: $<\TC,[\TC,\NC),\NC$. The difficulty for Transformers to handle these problems increases progressively. This difficulty classification can be extended to existing POMDP problems. Please refer to Appendix~\ref{sec:discuss-pomdp} for detailed discussion.} 

\begin{itemize}
    \item \cam{$<\TC$: Most tasks that solely assess pure memory capabilities are weaker than $\TC$. These tasks only involve extracting a finite number of tokens from the past and performing simple logical operations with current observation information. Most memory tasks mentioned in \citet{tfshine} fall into this category. $(\TF,\RL)$ excel at solving such problems.} 
    \item \cam{$[\mathsf{TC}^0,\mathsf{NC}^1)$: This category already represents the vast majority of regular languages. The corresponding typical POMDPs are environments such as Passive Visual Match~\cite{hung2019optimizing} or Memory Maze~\cite{pasukonis2022evaluating}, where there is a need to infer the current position based on historical information. This is typically manifested in the requirement to reconstruct a relatively simple state from complex historical data.}
    \item \cam{$\NC$: Currently, no existing discrete-state POMDP problem has been found to correspond to this class of regular languages. According to Theorem~\ref{thm:lim-in-tf-fit}, it is difficult for $(\TF,\RL)$ to learn the optimal policy.} 
\end{itemize}

\cam{Establishing a direct connection with regular languages is not particularly straightforward in continuous scenarios. However, some standard POMDP scenarios, such as Pybullet Occlusion Task~\cite{recurrentbaseline}, are at least not in the first level. These tasks require inferring the current actual state based on contextual information.}

\note{Furthermore, the preceding discussion implies that for $(\TF, \RL)$, the hidden state fed to $\RL$ is often not the underlying \textit{real state}. 
%\simon{did we define real state? maybe say underlying state?} 
In contrast, in subsequent experiments (cf. Section~\ref{sec:exp}), we observe that $(\texttt{RNN}, \RL)$ behaves differently and can implicitly reconstruct the \textit{real state}. The capability of recovering underlying \textit{real states} with the Markov property is believed to be a prerequisite for solving \porl. Therefore, \note{for POMDPs in general cases}, $(\TF, \RL)$ may encounter issues.}

% In continuous-state scenarios, while establishing a direct analogy with regular languages is not feasible, our experiments conducted on standard POMDP settings (cf. Section~\ref{exp:pomdp}) indicate that Transformers encounter challenges in reconstructing the underlying state based on historical information in continuous situations, ultimately resulting in subpar performance.

\section{Combining Transformer and RNN}
% From the analysis in Section~\ref{sec:lim-of-tf-in-porl}, it becomes evident that RNN-like models (LSTM, GRU, RNN), capable of solving Regular Language problems, emerge as promising sequence model choices for \porl. 
\begin{table*}[htbp]
\vspace{0.2in}
\centering
\caption{\textbf{The recurrent representation for Transformer variants with pointwise recurrence.} We compare different Transformer variants: FART, FWP, RWKV and RetNet. $y_i,u_i$ are as defined in Section~\ref{sec:pre}, $s_i,z_i$ are hidden states, and the other variables are parameters. } 
\vspace{0.2in}
\begin{adjustbox}{width=0.8\textwidth}
\setlength{\tabcolsep}{10mm}{
\newsavebox\fartrecurrence
\begin{lrbox}{\fartrecurrence}
  \begin{minipage}{\columnwidth}
    $y_i=\operatorname{FFN}\qty(\frac{\phi(u_iW_Q)^\top}{\phi(u_iW_Q)^\top z_i}+u_i),\begin{aligned}
        s_i&=s_{i-1}+\phi\qty(u_iW_k)(u_i W_V)^\top\\
        z_i&=z_{i-1}+\phi\qty(u_iW_k)\\
    \end{aligned}$ 
  \end{minipage}
\end{lrbox}

\newsavebox\rwkvrecurrence
\begin{lrbox}{\rwkvrecurrence}
  \begin{minipage}{\columnwidth}
    $y_i=\operatorname{Gate}\qty(\frac{s_{i-1}+\ee^{v+W_ku_i}\odot u_t}{z_{i-1}+\ee^{v+W_ku_i}}),\begin{aligned}
        s_i&=\ee^{-w}\odot s_{i-1}+\ee^{W_ku_i}\odot u_i\\
        z_i&=\ee^{-w}\odot z_{i-1}+\ee^{W_ku_i}\\
    \end{aligned}$ 
  \end{minipage}
\end{lrbox}

\newsavebox\fwprecurrence
\begin{lrbox}{\fwprecurrence}
  \begin{minipage}{\columnwidth}
    $y_i=\frac{1}{z_i\phi\qty(W_qu_i)}W_i\phi\qty(W_qu_i),\begin{aligned}
        W_i&=W_{i-1}+(W_vu_i)\otimes\phi\qty(W_ku_i)\\
        z_i&=z_{i-1}+\phi\qty(W_ku_i)\\
    \end{aligned}$ 
  \end{minipage}
\end{lrbox}

\newsavebox\retnetrecurrence
\begin{lrbox}{\retnetrecurrence}
  \begin{minipage}{\columnwidth}
    $y_i=\qty(\tau\qty(XW_G)\odot\operatorname{GN}(z_i)),z_i=\gamma z_{i-1}+(Ku_i)^\top(Vu_i)$
  \end{minipage}
\end{lrbox}

\begin{tabular}{lc}
\toprule
Architectures & Recurrent Representation for a Single Head \\
\midrule
FART~\cite{fart} & \usebox{\fartrecurrence}  \\
FWP~\cite{fwp} & \usebox{\fwprecurrence} \\
RWKV~\cite{rwkv} & \usebox{\rwkvrecurrence} \\
RetNet~\cite{retnet} & \usebox{\retnetrecurrence} \\
\bottomrule
\end{tabular}

}
\end{adjustbox}
\label{tab:current-tf}
\end{table*}

\note{
From the analysis in Section~\ref{sec:lim-of-tf-in-porl}, it becomes evident that RNN-like models (LSTM, GRU, RNN) emerge as promising sequence model choices for \porl. 
} \cam{There has been considerable theoretical work demonstrating their completeness on regular languages~\cite{merrill2019sequential,korsky2019computational}. For cases of log precision, based on definitions, we can directly map the recurrent units of RNNs to transition functions in DFAs. Therefore, RNNs do not suffer from the theoretical constraints encountered by Transformers.}

However, RNN-like models face the challenge of rapid memory decay~\cite{tfshine, gtrxl}, leading to an inferior performance on POMDP problems that demand long-term memory when compared to Transformers~\cite{gtrxl,tfshine}.

\note{Another insight from the previous section is that the attention mechanism of Transformers is primarily to blame for their limitations (see Figure~\ref{fig:attention}). As articulated in ~\citet{parallelism}, there exists a trade-off between the highly parallel structure of Transformers and their computational capacity.}
% \note{
% As articulated in ~\citet{parallelism}, there exists a trade-off between the highly parallel structure (see Figure~\ref{fig:attention}) of Transformers and their computational capacity. 
% \kz{(rephrase) On the other hand, we can also observe that the attention mechanism is crucial in proving the limitations of Transformer in the previous section, confirming this conclusion.}
% } 

\note{To \cam{alleviate} these limitations of Transformers, a natural idea is to endow Transformers with the ability of pointwise recurrence (see Figure~\ref{fig:recurrent}).} This line of development has been the focus of numerous efforts, resulting in several Transformer variants that incorporate this mechanism, as detailed in Table \ref{tab:current-tf}. % The commonality in their recurrence representation for a single head lies in performing simple linear operations, such as $x_t = \lambda x_{t-1} + u_t$,
The shared feature of these methods in their recurrence representation for a single head can be found in the simple linear operations they employ, such as $x_t = \lambda x_{t-1} + u_t$. While the non-linear components can be amortized across the layers through the FFN between layers. If the number of heads is set equal to the dimension of the hidden state $h$, then \begin{align}
    \vb{x}_t=\vb{\Lambda}\vb{x}_{t-1}+\vb{u}_t\;,\label{eq:linearattn}
\end{align}
where $\vb{\Lambda}=\operatorname{diag}\qty(\lambda_1,\ldots,\lambda_h)$. In RetNet, operations involving $x_t$, $\lambda_i$, and $u_i$ are performed over $\mathbb{C}$, while the remaining operations are carried out over $\RR$. As for RWKV, $\lambda_i$ is a learnable parameter, while in the rest variants are hyperparameters.

From the perspective of RNN, if we linearize and diagonalize the RNN's recurrent unit $\vb{x}_t=\vb{A}\vb{x}_{t-1}+\vb{B}\vb{u}_t$, we obtain the following form: \begin{align}
\vb{\tilde{x}}_t=\vb{\Lambda}\vb{\tilde{x}}_{t-1}+\vb{\tilde{u}}_t\label{eq:linearrnn}\;,
\end{align}
where $\vb{A}=\vb{P}\vb{\Lambda}\vb{P}^{-1},\vb{\tilde{x}}_t=\vb{P}^{-1}\vb{x}_t,\vb{\tilde{u}}_t=\vb{P}^{-1}\vb{B}\vb{u}_t$. Since almost all matrices can be diagonalized over $\mathbb{C}$, the operations mentioned above are defined in $\mathbb{C}$ \cite{probability}.

Comparing (\ref{eq:linearattn}) and (\ref{eq:linearrnn}), the pointwise-recurrence Transformer can be viewed as a linear RNN with certain constraints, and the linear RNN serves as a balance point between Transformers and RNNs. To summarize, we expect linear RNN to be more suitable as a sequence model in \porl for the following reasons. 

\textbf{Regular language.} Many studies suggest that the recurrence with non-linear activation functions plays a crucial role in the completeness of RNNs in regular languages~\cite{turing}, while linear RNNs may lose this completeness. However, some researches indicate that linear RNNs can effectively approximate RNNs \cite{rsa,parallelizing} and perform well on formal language tasks similar in form to NLP \cite{rsa,irie2023practical}. Compared to Transformers, their inductive biases are closer to HMMs. In subsequent experiments (cf. Section~\ref{sec:exp}), we validate that $(\texttt{LRNN}, \RL)$ can implicitly learn the states in POMDPs.

\textbf{State space model.} Linear RNNs have been proven to efficiently fit partially observable \textit{linear dynamic systems}~\cite{lti}. \note{While the transformer's fitting capability has theoretical proofs only under certain specific conditions}~\cite{tf-lti,tf-algo}, with no similar conclusion for more general situations. The linear dynamic system can be considered as a first-order approximation of a state space model, indicating the potential of linear RNNs in addressing a broader range of POMDPs.

\textbf{Long term memory.} The primary reason for the long-term dependency issues in RNNs is the challenge of gradient explosion or vanishing when input length increases during training~\cite{rnn-diffult}. Transformers, due to their parallel structure, are less susceptible to this issue. To mitigate this problem, gate mechanisms are introduced to RNNs~\cite{gru,lstm}. However, \citet{prevent} indicates that the non-linear recurrence is the primary cause of gradient explosions. For linear RNNs, effectively managing the range of parameters $\lambda_i$ between $[0,1]$ during initialization successfully addresses both gradient explosion and vanishing issues. This has been validated in certain supervised learning tasks with long-term dependencies~\cite{linearrnn,s4}.

% \textbf{}

% The motivation behind these works is typically to enhance inference efficiency. 

% Here, we highlight the potential advantages of these works in overcoming the theoretical limitations of Transformers. \td{add}

\section{Experiments}
\label{sec:exp}
In this section, we compare the effectiveness of three different sequence models --- Transformer, RNN, and \note{linear RNN} --- in addressing partially observable decision-making problems within the realm of 
% partially observable reinforcement learning (RL) methods 
\note{\porl}
(\texttt{SEQ}, \texttt{RL}). 
% For the Transformer model, we utilize the widely acclaimed GPT architecture~\citep{GPT2}, while for the RNN, we employ the LSTM model~\citep{lstm}. 
\note{We choose GPT~\citep{GPT2}, LSTM~\citep{lstm}, and LRU~\citep{linearrnn} as the representative architectures for these three types of models.}
To substantiate our hypotheses, we conduct experiments in three distinct POMDP scenarios, detailed in Sections~\ref{exp:r} to \ref{exp:long-term}. These experiments are designed to assess the models from various perspectives: 1) POMDPs derived from certain regular languages, including \texttt{EVEN PAIRS}, \parity, and \sfive; 2)
\note{tasks from Pybullet Partially Observable environments~\citep{recurrentbaseline} that require the ability of state space modeling;}
% tasks from Pybullet Partially Observable environments that require the ability of state space modeling, including Occlusion~\citep{recurrentbaseline} as well as Observability and Constructability; 
3) tasks that require pure long-term memory capabilities, such as Passive T-Maze and Passive Visual Match~\cite{longtimememory}. Comprehensive implementation details, task descriptions, and supplementary results are presented in Appendix~\ref{sec:impl}. \cam{We also conduct a comparison with some published Transformer in RL there.}

% \textbf{We are committed to releasing the code}.

% \subsection{Evaluation on POMDPs from Reducing Regular Languages}
\subsection{\note{POMDPs Derived from Regular Languages}}
\label{exp:r}
% Results for regular language tasks are shown in Table~\ref{tab:regular-task}. We select sequence models including LRU, Transformer (denoted as GPT), and LSTM, plugging them into the RL algorithms. Generally, all three sequence models can fit scenarios with short lengths. However, as the length increases, LSTM exhibits the strongest fitting capability, followfed by LRU, and GPT performs the least effectively. To look into how these sequence models model the regular languages, we visualize the hidden states in Figure~\ref{fig:regular-hidden}. We observe that in these POMDP tasks determined by three regular languages, the distinct nature of these languages yields varied results:

\note{
We construct this type of POMDP problem following the approach of Proposition~\ref{prop:r-to-pomdp}, and use DQN~\cite{ddqn} as the RL component in $(\SEQ, \RL)$. \cam{Three regular language tasks correspond to the difficulty classification in Section \ref{sec:pomdp-to-regular}.} 
} The learning curves are shown in Appendix~\ref{appendix:supplementary}, Figure~\ref{fig:regular-learning-curve}. \cam{We provide the experimental results on length extrapolation and model scale for this task in the appendix, where
Theorem \ref{thm:lim-in-tf-fit} and Theorem \ref{thm:lim-in-tfrl-gen} are validated.
} 

To look into how they model the regular languages, we visualize the hidden states in Figure~\ref{fig:regular-hidden}.  Generally, all three sequence models can fit scenarios with short lengths. However, as the input length increases, LSTM exhibits the best fitting capability, followed by LRU, and GPT performs the least effectively. We observe that in POMDP tasks derived from the three regular languages, the distinct nature of these languages yields varied results:
% \begin{table}[htbp]
% \centering
% \caption{\textbf{Returns for regular language tasks}. We compare different sequence models: LRU, GPT, and LSTM. \kz{We don't need this table, elaborate more on the results and refer them to the curve in appendix}}
% \vspace{0.1in}
% \begin{adjustbox}{width=\columnwidth}
% \input{tables/regular_result_table}
% \end{adjustbox}
% \label{tab:regular-task}
% \end{table}

\begin{figure}[ht]
\vskip 0.1in
\caption{\textbf{Hidden state for regular language tasks}. We visualize the hidden states of each sequence model during evaluation at length 25 using t-SNE~\cite{tsne} and annotate them according to their real states. Our classification corresponds to the state the observation history maps to in the reduced POMDP, namely $(q, w)$, while `T' stands for the terminal state. \note{The states with similar colors in the diagram generally produce the same type of observation.}}
\vskip 0.0in
\begin{center}
\subfloat[\texttt{EVEN PAIRS}]{
    \includegraphics[width=\columnwidth]{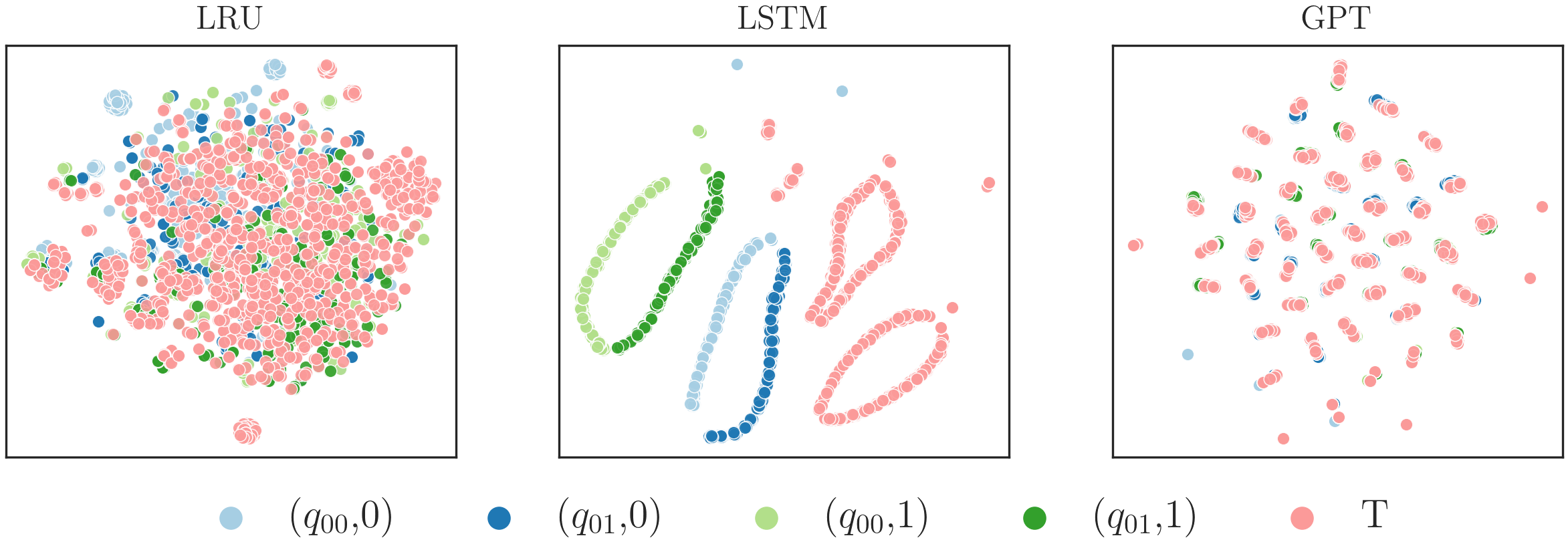}
}
\newline
\subfloat[\parity\label{fig:regular-hidden:parity}]{
    \includegraphics[width=\columnwidth]{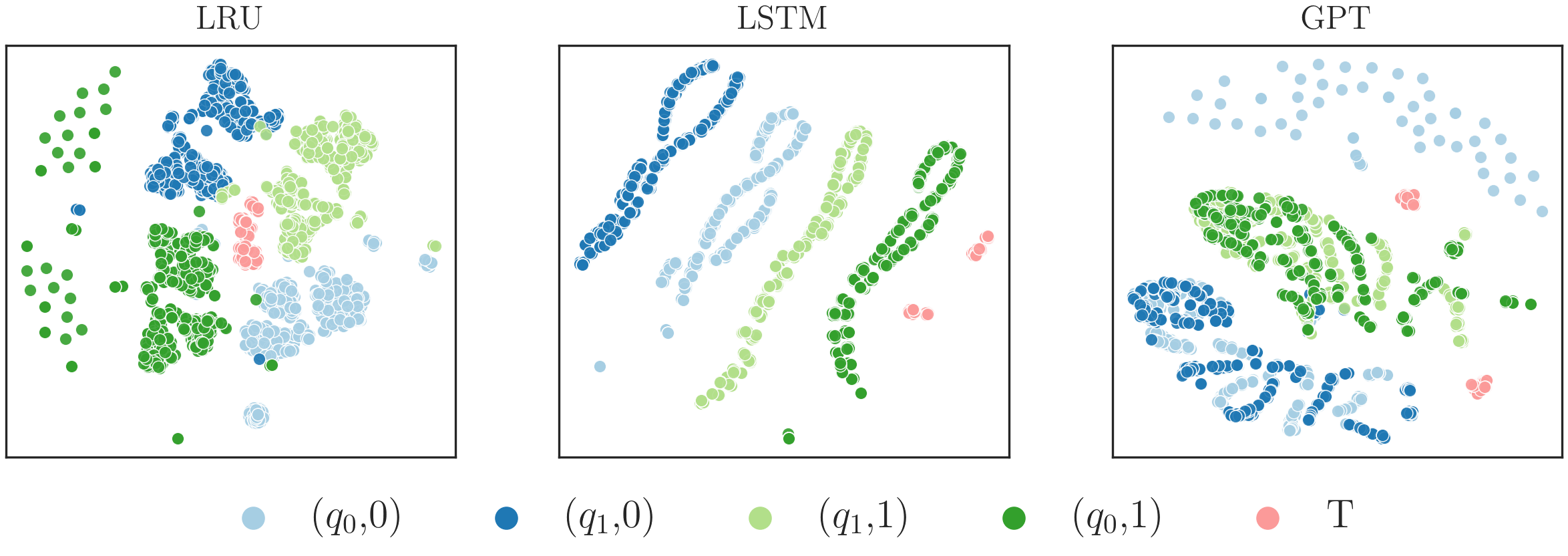}
}
\newline
\subfloat[\sfive]{
    \includegraphics[width=\columnwidth]{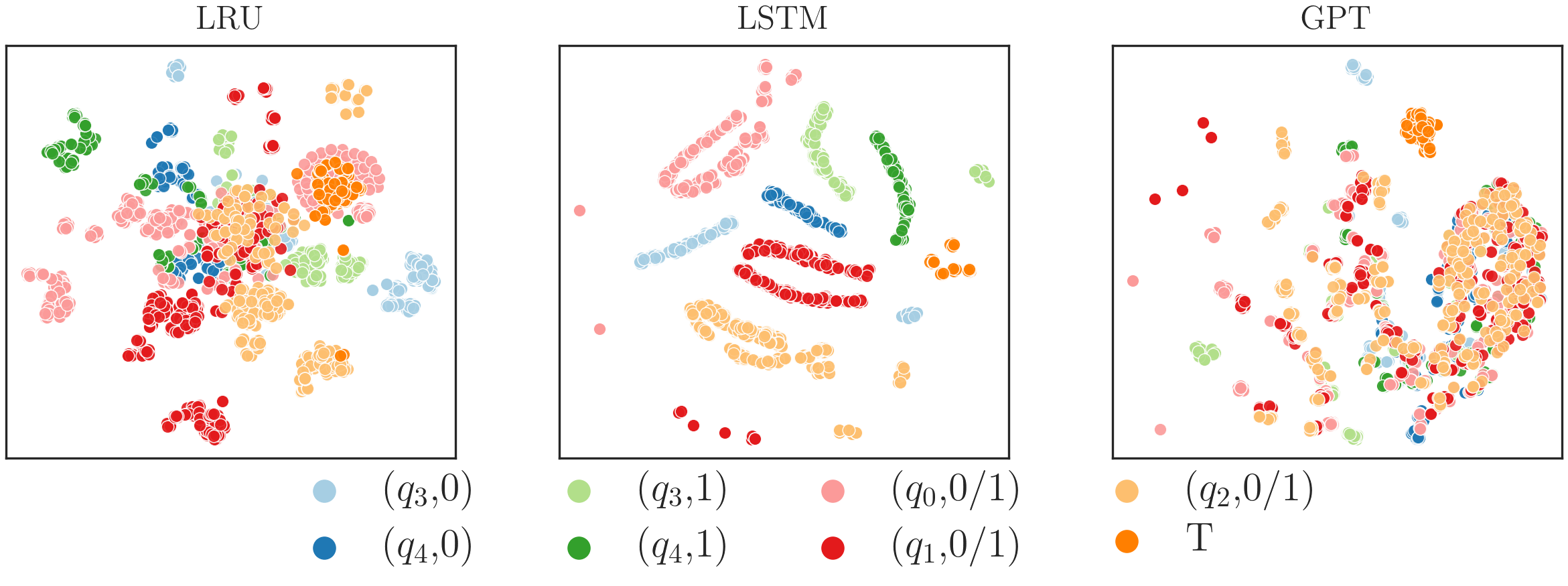}
}
\label{fig:regular-hidden}
\end{center}
\vskip -0.5in
\end{figure}

\textbf{\texttt{EVEN PAIRS}} is a specific regular language that could be directly solved by memorizing the first character and comparing it with the last character, which aligns with the inductive bias of the attention mechanism. As a result, GPT solves $\mathcal M^\text{\texttt{EVEN PAIRS}}$ reasonably well.

\textbf{\parity} \note{is a regular language with simple DFA in $\TC$.}
% is a standard regular language that requires simulating the state space of the DFA in Figure~\ref{fig:reduction-parity} to solve $\mathcal M^\parity$, otherwise demands exponential memory space. 
As shown in Figure~\ref{fig:regular-hidden:parity}, LSTM and LRU are capable of accurately modeling $\mathcal M^\parity$. \note{Through colors, it can be observed that the hidden state of the transformer is almost solely distinguished based on the current observation. It relies on processing the entire history through attention after encountering a terminal symbol. This is more like memorizing all the different strings, resulting in lower final returns.}
% \kz{In contrast, GPT relies on memorizing the entire historical sequence and the corresponding outcomes, while LRU falls between the two. Empirically, LSTM outperforms both LRU and GPT.}

\textbf{\sfive} is a $\NC$ complete regular language as mentioned in Section~\ref{subsec:fitting}, and we have shown the inability of GPT to solve $\mathcal M^{\sfive}(n)$ in Theorem~\ref{thm:lim-in-tf-fit}. Experimental results align with our claim, proving that GPT performs worst in this task and fails to recover the true state.

\cam{}

% \begin{figure}[ht]
% \vskip 0.2in
% \begin{center}
% \subfloat[Even Pairs]{
%     \includegraphics[width=0.8\columnwidth]{plts/even_pairs-regular-return-agg.pdf}
% }
% \newline
% \subfloat[Parity]{
%     \includegraphics[width=0.8\columnwidth]{plts/parity-regular-return-agg.pdf}
% }
% \newline
% \subfloat[S$_5$]{
%     \includegraphics[width=0.8\columnwidth]{plts/s5-regular-return-agg.pdf}
% }
% \caption{\textbf{Aggregation via length for regular language tasks}.}
% \label{fig:regular-agg}
% \end{center}
% \vskip -0.2in
% \end{figure}

% \subsection{Evaluation on PyBullet Partially Observable Environments}

\begin{figure*}[htbp]
\vskip 0.2in
\begin{center}
\centerline{\includegraphics[width=0.8\textwidth]{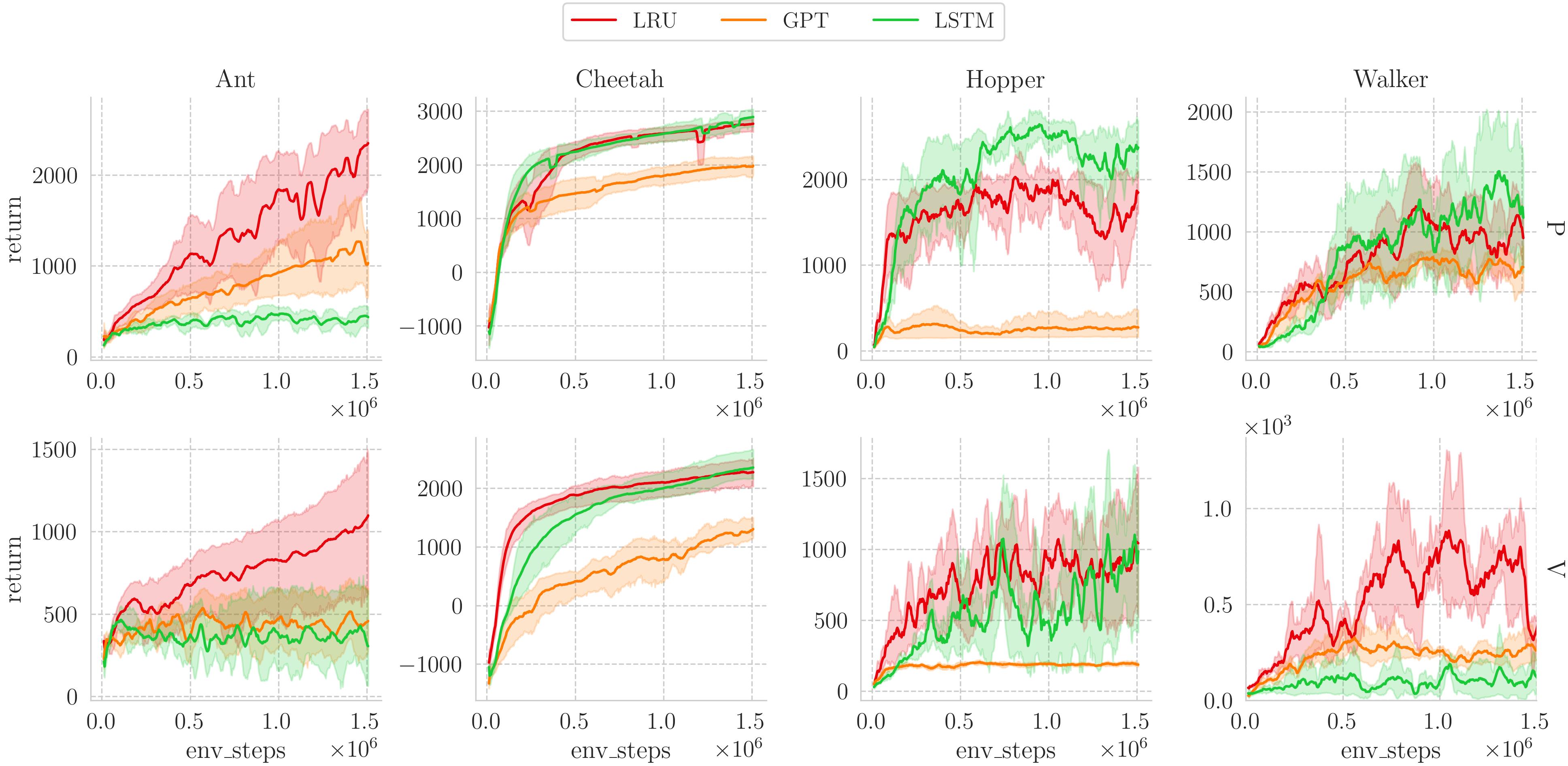}}
\caption{\textbf{Learning curves for PyBullet occlusion tasks}. Mean of 5 seeds. The shaded area indicates $95$\% confidence intervals.}
\label{fig:pomdp-learning-curve}
\end{center}
\vskip -0.1in
\end{figure*}

\begin{table}[htbp]
\centering
\caption{\textbf{Normalized scores for PyBullet occlusion tasks}. We compare different sequence models LRU, GPT and LSTM. `V' refers to `only velocities observable', and `P' refers to `only positions observable'. We present normalized scores defined in Appendix~\ref{appendix:supplementary}, Equation~\ref{eq:normalized_score}. \mhl{Blue} highlight indicates the highest score, and \mlhl{orange} highlight indicates the second-highest score.}
\vspace{0.1in}
\begin{adjustbox}{width=\columnwidth}
\begin{tabular}{cccccc}
\toprule
Task & Type & LRU & GPT & LSTM \\ \midrule

\multirow{2}{*}{Ant} & V & \mhl{\po\po\po 29.8}\textcolor{gray}{\mhl{ $\pm$ 20.4\po\po\po }} & \mlhl{\po\po\po\po 9.4}\textcolor{gray}{\mlhl{ $\pm$ 7.1\po\po\po\po }} & {\po\po\po\po 7.4}\textcolor{gray}{{ $\pm$ 8.8\po\po\po\po }} \\
 & P & \mhl{\po\po\po 81.2}\textcolor{gray}{\mhl{ $\pm$ 28.7\po\po\po }} & \mlhl{\po\po\po 38.2}\textcolor{gray}{\mlhl{ $\pm$ 26.0\po\po\po }} & {\po\po\po\po 5.7}\textcolor{gray}{{ $\pm$ 3.3\po\po\po\po }} \\
\multirow{2}{*}{Cheetah} & V & \mlhl{\po\po\po 96.8}\textcolor{gray}{\mlhl{ $\pm$ 8.1\po\po\po\po }} & {\po\po\po 69.3}\textcolor{gray}{{ $\pm$ 6.9\po\po\po\po }} & \mhl{\po\po\po 98.1}\textcolor{gray}{\mhl{ $\pm$ 8.3\po\po\po\po }} \\
 & P & \mlhl{\po\po 109.9}\textcolor{gray}{\mlhl{ $\pm$ 4.2\po\po\po\po }} & {\po\po\po 88.8}\textcolor{gray}{{ $\pm$ 6.3\po\po\po\po }} & \mhl{\po\po 112.5}\textcolor{gray}{\mhl{ $\pm$ 5.4\po\po\po\po }} \\
\multirow{2}{*}{Hopper} & V & \mhl{\po\po\po 94.1}\textcolor{gray}{\mhl{ $\pm$ 23.4\po\po\po }} & {\po\po\po 13.5}\textcolor{gray}{{ $\pm$ 0.5\po\po\po\po }} & \mlhl{\po\po\po 82.5}\textcolor{gray}{\mlhl{ $\pm$ 37.1\po\po\po }} \\
 & P & \mlhl{\po\po 147.9}\textcolor{gray}{\mlhl{ $\pm$ 12.3\po\po\po }} & {\po\po\po 23.8}\textcolor{gray}{{ $\pm$ 18.0\po\po\po }} & \mhl{\po\po 184.1}\textcolor{gray}{\mhl{ $\pm$ 13.4\po\po\po }} \\
\multirow{2}{*}{Walker} & V & \mhl{\po\po\po 61.7}\textcolor{gray}{\mhl{ $\pm$ 14.6\po\po\po }} & \mlhl{\po\po\po 22.3}\textcolor{gray}{\mlhl{ $\pm$ 6.0\po\po\po\po }} & {\po\po\po 12.2}\textcolor{gray}{{ $\pm$ 7.0\po\po\po\po }} \\
 & P & \mlhl{\po\po\po 79.3}\textcolor{gray}{\mlhl{ $\pm$ 23.1\po\po\po }} & {\po\po\po 49.5}\textcolor{gray}{{ $\pm$ 4.8\po\po\po\po }} & \mhl{\po\po\po 94.6}\textcolor{gray}{\mhl{ $\pm$ 36.6\po\po\po }} \\
\midrule
\multicolumn{2}{c}{\textbf{Average}} & \mhl{\po\po88.2\po\po} & {\po\po39.3\po\po} & {\po\po74.6\po\po} \\
\bottomrule
\end{tabular}
\end{adjustbox}
\label{tab:pomdp-task}
\end{table}

\begin{figure}[htbp]
\vspace{0.1in}
\begin{center}
\includegraphics[width=\columnwidth]{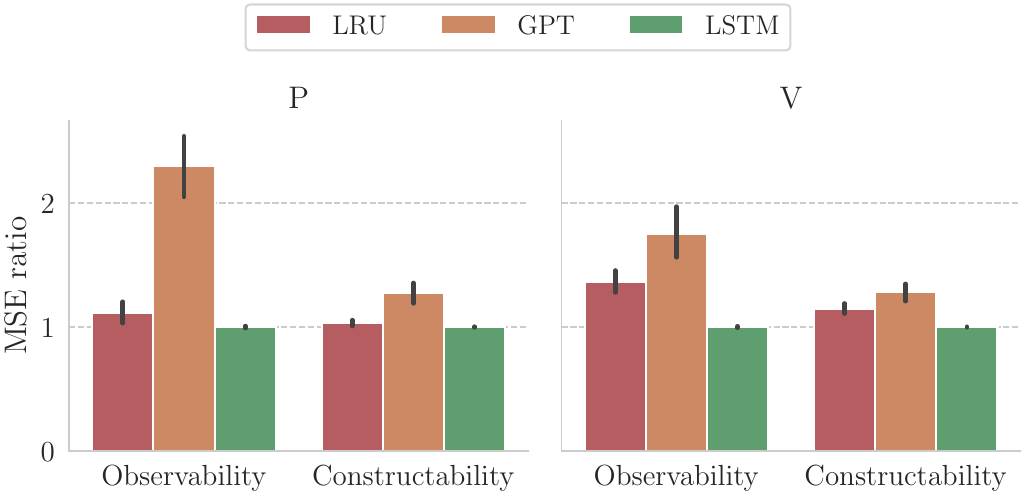}
\caption{\textbf{Mean Squared Error (MSE) ratios for \ssm tasks}. `V' refers to `only velocities observable', and `P' refers to `only positions observable'. To enable a comparative analysis of the performance among the three models, we present the MSE ratio, as defined in Appendix~\ref{appendix:supplementary}, Equation~\ref{eq:mse_ratio}.}
\label{fig:ssm}
\end{center}
\vskip 0in
\end{figure}

\subsection{\note{PyBullet Partially Observable Environments}}
\label{exp:pomdp}
We conduct experiments on 8 partially observable environments, which are all PyBullet locomotion control tasks with parts of the observations occluded~\citep{recurrentbaseline}, and denote them as PyBullet Occlusion. These experiments encompass four distinct tasks: Ant, Cheetah, Hopper, and Walker, and we evaluate the models based on two types of observations: Velocities Only (V) and Positions Only (P). The normalized scores are demonstrated in Table~\ref{tab:pomdp-task}, and we also provide learning curves in Figure~\ref{fig:pomdp-learning-curve}. From the results, it is evident that LRU and LSTM outperforms GPT in all eight tasks, matching our claim that the Transformer architecture struggles at modeling partially observable sequences. \cam{The results showing that LSTM outperforms GPT are also verified in \citet{tfshine}.} 

Moreover, the general performances of LRU and LSTM are notably comparable, and LRU significantly outperforms LSTM in certain tasks, namely Ant (P, V), and Walker (V). Such results demonstrate that after linearization, recurrent-based models can still effectively retain their capacity to model the sequence, and can serve as a well-rounded balance integrating the strengths of both Transformer and RNN architectures.

We conduct ablation experiments with full observability in Appendix~\ref{appendix:supplementary}, Table~\ref{tab:pomdp-task-fully-obs}, and the overall performances of the three models are close, affirming that GPT's inferior performance in POMDP scenarios stems from partial observability rather than other factors.

\begin{figure*}[htbp]
\vskip -0.1in
\begin{center}
\subfloat[Passive T-Maze]{
    \includegraphics[width=0.8\columnwidth]{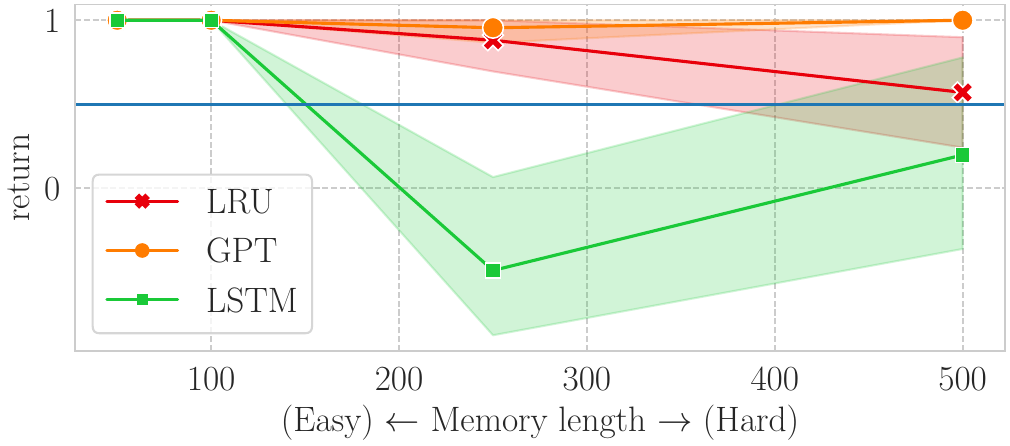}
}
\subfloat[Passive Visual Match]{
    \includegraphics[width=0.8\columnwidth]{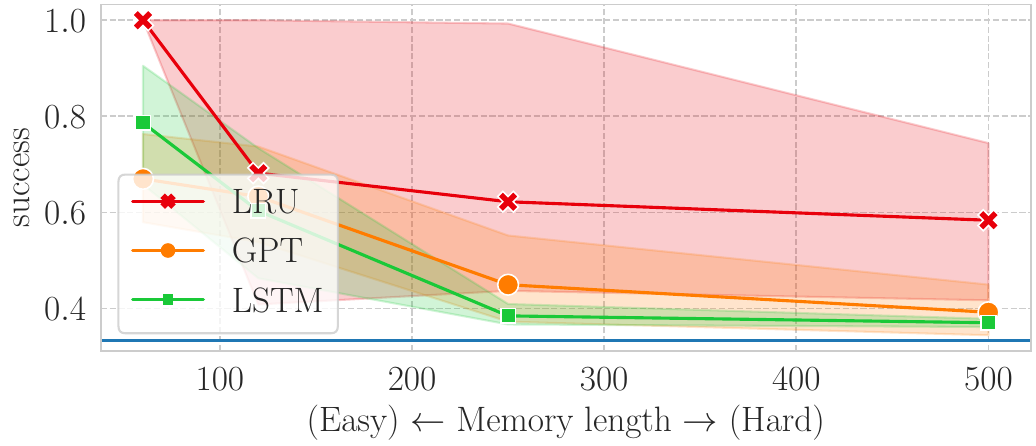}
}
\caption{\textbf{Results in pure long-term memory environments with varying memory lengths}. \textcolor{blue}{Blue} line indicates the return of optimal markov policy, which only has access to the observation.}
\label{fig:long-agg}
\end{center}
\vskip -0.1in
\end{figure*}

To enhance our understanding of the capability to extract state information from observation sequences, we meticulously crafted two tasks. These tasks are aimed at determining the initial state, termed ``Observability'', and forecasting the current state, referred to as ``Constructability'', using historical observation sequences, and we adopt Mean Square Error (MSE) as our training target. Our experiments were conducted on the D4RL medium-expert dataset~\cite{D4RL} of the aforementioned tasks, and the results (illustrated in Figure~\ref{fig:ssm}) are presented as the average MSE ratios across these tasks. The findings reveal that, in both scenarios, GPT is notably less competent compared to the other two models. In contrast, the LRU model demonstrates capability on par with the LSTM model. This observation lends further support to our hypothesis that GPT's ability to reconstruct states from partially observable sequences is worse than that of the recurrent-based models.

% \subsection{Evaluation on Pure Long-term Memory Environments}
\subsection{\note{Pure Long-term Memory Environments}}
\label{exp:long-term}
Results for pure long-term memory environments, namely Passive T-Maze and Passive Visual Match, are provided in Figure~\ref{fig:long-agg}, and learning curves are shown in Appendix~\ref{appendix:supplementary}, Figure~\ref{fig:memory-learning-curve}. In these experiments, we follow the work of~\citet{tfshine}, which tests the long-term memory ability of Transformer-based agent and LSTM-based agent on two memory-demanding tasks. We observe that LRU performs comparably to GPT, while significantly outperforming LSTM. Furthermore, LRU beats GPT on Passive Visual Match, the harder task of the two which involves a complex reward function~\cite{longtimememory}, showcasing its powerful long-term memory capability.

\section{Conclusion}
In this work, we challenge the suitability of Transformers as sequence models in \porl. Through theoretical analysis and empirical evidence, we reveal Transformer's limitations in solving POMDPs, particularly their struggle with modeling regular languages, a key aspect of POMDPs. As a remedy to these issues, We propose LRU as a more effective alternative, combining the strengths of recurrence and attention. Supported by extensive experiments, our findings challenge the prevailing use of Transformers in sequential decision-making tasks, and open new avenues for exploring recurrent structures in complex, partially observable environments.

It is also important to acknowledge the limitations of our work. After introducing recurrence, LRU serves as a choice to combine the advantages of Transformer and RNN, while still lacking theoretical guarantees for modeling regular languages. Although LRU demonstrates satisfactory performance in experiments, there remains a need for further exploration in this direction. \cam{Additionally, the theoretical analysis in this paper focuses more on the exploitation aspect of RL, while lacking discussion on exploration. Complex POMDP tasks not only require suitable sequence models but also need to be paired with appropriate RL algorithms.}

% GPT version: In conclusion, our research critically evaluates the suitability of Transformers for POMDPs in reinforcement learning. Through theoretical analysis and empirical evidence, we reveal the inherent limitations of Transformers in these scenarios, particularly their struggle with modeling regular languages, a key aspect of POMDPs. We propose the Deep Linear Recurrent Unit (LRU) as a more effective alternative, combining the strengths of recurrence and linear processing. Our findings challenge the prevailing use of Transformers in sequential decision-making tasks and open new avenues for exploring recurrent structures in complex, partially observable environments.

% \clearpage
\section*{Impact Statement}
Our work revisits the application of Transformers in RL, aiming to advance the development of decision intelligence. If misused in downstream tasks, it has the potential to lead to adverse effects such as privacy breaches and societal harm. Nevertheless, this is not directly related to our research, as our primary focus is on theoretical investigations.
% \clearpage
% \newpage
\bibliography{icml2024}
\bibliographystyle{icml2024}
\newpage
\appendix
\onecolumn
\section{Additional Background and Notation}
\subsection{Circuit complexity}
\label{sec:circuit}
In this subsection, We introduce several basic complexity classes, namely $\AC$, $\TC$, and $\NC$:
\begin{itemize}
    \item $\AC$ contains all languages that are decided by Boolean circuits with constant depth, unbounded fan-in, and polynomial size, consisting of AND gates, OR gates, NOT gates;
    \item $\TC$ is $\AC$ with majority gates which outputs true if and only if more than half of the input bits are true;
    \item $\NC$ contains all languages that are decided by Boolean circuits with a logarithmic depth of $\mathcal O(\log n)$ where $n$ is the input length, constant fan-in, and polynomial-size, consisting of AND gates, OR gates, and NOT gates.
\end{itemize}

The relationships between them are $\AC\subseteq\TC\subseteq \NC$, and it is commonly conjectured that $\TC\ne\NC$ whereas it remains an open problem in the computation complexity theory. A language $L\in \NC$ is $\NC$ complete w.r.t. $\AC$ reduction if for any $L'\in \NC$, $L'\le_{\text{strong}}L$, \textit{i.e.} $L'$ is reducible to $L$ under $\AC$ reduction. More details can be referred to \citet{semigroups}.
\subsection{$\NC$ complete regular language}
In this subsection, we introduce the approach of connecting regular languages and $\NC$ complete problems using the syntactic monoid theory and Barrington's theorem.

\subsubsection{Syntactic monoid}
The \textit{syntactic monoid} is a concept in the algebraic language theory that establishes a connection between the language recognition and the group theory.

\begin{definition}[Syntactic congruence \cite{semigroups}] Let $A$ be a finite alphabet, and let $L\subseteq A^*$. We define an equivalence relation $\equiv_L$ on $A^*$: $x\equiv_L y$ iff. 
\[
\qty{(u,v)\in A^*\times A^*:uxv\in L}=\qty{(u,v)\in A^*\times A^*:uyv\in L}\;.
\]
Note that $xa\equiv_L ya,ax\equiv_L ay$ if $x\equiv_L y,a\in A$, it follows that $\equiv_L$ is a congruence on $A^*$, called the \textit{syntactic congruence}.
\end{definition}

\begin{definition}[Syntactic monoid~\cite{semigroups}] 
Given a language $L\subseteq A^*$, the quotient of $A^*$ by its congruence $\equiv_L$ is called the \textit{syntactic monoid} of $L$ and is denoted as $M(L)$.
\end{definition}

For a regular language $L$, determining $M(L)$ can be accomplished using a straightforward method \cite{syntactic}. The procedure involves initially computing its minimal DFA, with the syntactic semigroup of $L$ being equivalent to the transition semigroup $S$ of the DFA.

\subsubsection{Barrington’s Theorem}
\citet{barrington} demonstrated that the word problem of the group $S_5$ is $\NC$ complete. The word problem of a group $G$ is defined as $\qty{g_1 \ldots g_n = e: g_i \in G}$. The following theorem offers a comprehensive statement of Barrington's work.

\begin{theorem}[Theorem IX.1.5 in \citet{semigroups}]
    \label{thm:barrington}
    Given a regular language such that $M(K)$ is not solvable. Then for all $L\in\NC, L\le_\text{strong} K$.
\end{theorem}

The methods used in the reduction process are simpler than $\NC$; specifically, they involve employing $\AC$ or $\TC$ for the reduction.
The well-known connection between this theorem and the original word problem of $S_5$ is as follows: for $n \ge 5$, the symmetric group $S_n$ is unsolvable.

\subsubsection{Examples of $\NC$ complete Regular Language}
\label{example:nc-complete}

\begin{proposition}
If $L=\nccomplete$, then $L$ is $\NC$ complete.
\end{proposition}

\begin{figure}[htbp]
    \centering
    \begin{tikzpicture}[shorten >=1pt,node distance=2cm,on grid,auto]
  \tikzstyle{every state}=[fill={rgb:black,1;white,10}]

    \node[state,initial,accepting]   (q_0)                    {$q_0$};
    \node[state] (q_1)  [right of=q_0]    {$q_1$};
    \node[state] (q_2)  [right of=q_1]    {$q_2$};
    \node[state] (q_3)  [right of=q_2]    {$q_3$};
    \node[state] (q_4)  [right of=q_3]    {$q_4$};

    \path[->]
    % (q_1) edge [loop above] node {0}    (   )
    %       edge [bend left]  node {1}    (q_2)
    % (q_2) edge [bend left]  node {0}    (q_3)
    %       edge [loop above] node {1}    (   )
    % (q_3) edge [bend left]  node {0,1}  (q_2);
    (q_0) edge node {0, 1} (q_1)
    (q_1) edge node {0, 1} (q_2)
    (q_2) edge node {0, 1} (q_3)
    (q_3) edge node {0} (q_4)
          edge [bend left] node {1} (q_0)
    (q_4) edge [loop below] node {1} ( )
          edge [bend right] node {0} (q_0);
\end{tikzpicture}
    \caption{The minimal DFA of $\nccomplete$}
    \label{fig:dfa-nc1-complete}
\end{figure}

\begin{proof}
Let $f_w:Q\to Q$ represents an element in the transition group $L$, where $f_w(q)$ denotes reaching the node $f_w(q)$ after inputting the string $w$ at node $q$. As illustrated in Figure \ref{fig:dfa-nc1-complete}, the transition group contains the following elements: \[
    f_0=\mqty(0&1&2&3&4\\1&2&3&4&0),\ f_1=\mqty(0&1&2&3&4\\1&2&3&0&4).
\]
Then $f_1^{-1}=f_1^3=f_{111}$ and $f_0^{-1}=f_0^4=f_{0000}$. Note that $\qty(0\ 1\ 2\ 3\ 4)$ and $\qty(0\ 1\ 2\ 3)$ are the generators of $S_5$ so $M(L)=S_5$ is not solvable. According to Theorem~\ref{thm:barrington}, $L$ is $\NC$ complete.
\end{proof}

\section{Theoretical Results}
\subsection{Proof of Proposition \ref{prop:r-to-pomdp}}
\label{subsec:proof-r-to-pomdp}
\begin{proof}[Proof of Proposition \ref{prop:r-to-pomdp}]
This proof is based on construction. Given a regular language $L\subseteq\Sigma^*$. We insert an end symbol $\#\notin\Sigma$ to obtain a new regular language $L^\#=\qty(Q,\Sigma\cup\qty{\#},\delta,F,q_0)$ s.t. $w\in L$ iff. $w\#\in L^\#$. Construct a POMDP $\mathcal M=(S,A,T,R,\Omega,O,\gamma)$. The state space $S$ is $Q\times \left(\Sigma\cup\{\#\}\right)$. The action space $A$ is $\qty{\text{\textit{accept}},\text{\textit{reject}}}$ and the observation space $\Omega$ is the alphabet $\Sigma\cup\qty{\#}$. The initial state is $(q_0,w_0)$, where $w_0$ is randomly sampled from $\Sigma\cup\{\#\}$. Given a state $(q_t,w_t)$ at timestep $t$, the agent could observe the character $w_t$. If $w_t\neq\#$, the next state would be $(q_{t+1},w_{t+1})$ where $q_{t+1}=\delta(q_t,w_t)$ and $w_{t+1}$ is randomly sampled from $\Sigma\cup\{\#\}$, and the agent would receive no reward; If $w_{t}=\#$, the process would terminate, and the agent would receive a reward of $1$ if it correctly outputs the acceptance of $w=w_0\ldots w_{t-1}$ in $L$. Note that the optimal policy $\pi$ on $\mathcal M$ is to \textit{accept} all $w\in L$ and \textit{reject} all $w\notin L$, so if an algorithm $\mathcal A$ can solve POMDP problems, then $\mathcal A$ can recognize $L$. 
\end{proof}

\subsection{Proof of Theorem \ref{thm:lim-in-tf-fit}}
\label{subsec:proof-lim-in-tf-fit}
\begin{lemma}[Theorem 2 in \cite{parallelism}]
\label{lemma:tf-upperbound}
Given an integer $d$ and polynomial $Q$, any log-precision transformer with depth $d$ and hidden size $Q(n)$ operating on inputs in $\Sigma^n$ can be simulated by a logspace-uniform threshold circuit family of depth $3+(9+2d_\oplus)d$.
\end{lemma}
\begin{remark}
\label{remark:tf-upperbound}
The scope outlined by Lemma \ref{lemma:tf-upperbound} for Transformers is quite broad, as its description of FNNs allows for any log-precision function. Therefore, in the case of a log-precision $(\TF, \RL)$ algorithm, we can distribute the $\RL$ part across the last FNN layer of the original Transformer, treating the entire model as a single Transformer.
\end{remark}

\begin{proof}[Proof of Theorem \ref{thm:lim-in-tf-fit}]
Proof by contradiction. Suppose there exists an integer $d$ and polynomial $Q$ such that for any $n$, a log-precision $\mathcal A=(\TF,\RL)$ with depth $d$ and hidden size $Q(n)$ can solve $\ml{L}$, where $L$ is a $\NC$ complete regular language. 
Given $w\in\Sigma^*$, The algorithm $\mathcal A$ can determine the validity of $w\in L$ by checking whether the action output by $\mathcal A(w)$ is ``accept''. Consequently, $\mathcal A$ can solve an $\NC$ complete problem.

At the same time, as stated in Remark \ref{remark:tf-upperbound}, we can treat $(\TF,\RL)$ as a single Transformer. Based on Lemma \ref{lemma:tf-upperbound}, $\mathcal A$ can be interpreted as a logspace-uniform threshold circuit family of constant depth, indicating that $L\in\TC$. Since we assume $\TC\neq\NC$, the existence of such an algorithm $\mathcal A$ is not possible.
\end{proof} 

\subsection{Proof of Theorem \ref{thm:lim-in-tfrl-gen}}
\label{subsec:lim-in-tfrl-gen}
\begin{lemma}
\label{lemma:lim-in-tfrl-gen}
Given an integer $n$, a symbol $a\in\Sigma$, a regular language $L\subseteq\Sigma^*$, let $L[n]=\qty{x\in L:\abs{x}=n}$ and \[
    P_{n,a}=\qty{(xa,x'a)\in L[n+1]\times(\bar{L})[n+1]:d(x,x')=1},
\]
where $d(\cdot,\cdot)$ denotes the number of different symbols in $x$ and $x'$.
If $0<c(n,a)<\abs{\Sigma}^n$, then $\abs{P_{n,a}}>0$.
\end{lemma}
\begin{proof}
Suppose $\abs{P_{n,a}}=0$. Note that if $xa\in L[n+1]$, then $\qty{x'a:d(x,x')=1}\subseteq L[n+1]$. Repeating this deduction, we can cover all $x\in\Sigma^n$. Hence, $\Sigma^na\subseteq L[n+1]$, that is, $c(n,a)=\abs{\Sigma}^n$. If $xa\notin L[n+1]$, then $xa\in (\bar{L})[n+1]$. As stated above, $c(n,a)=0$. By contradiction, $P_{n,a}>0$.
\end{proof}
\begin{proof}[Proof of Theorem \ref{thm:lim-in-tfrl-gen}]
Since $\Sigma$ is finite, $D=\max_{a,b\in\Sigma}\norm{u(a)-u(b)}$ where $u$ denotes the embedding vector of the given symbol. By Lemma~\ref{lemma:lim-in-tf-gen}, there exists $n$ such that $\qty{n:\abs{P_{n,a}}>0}$ is infinte.

Therefore, there exists infinite sequences $u$ and $u'$ such that $u$ and $u'$ differ by only $1$ position, yet $u\in L$ while $u'\notin L$. According to Lemma \ref{lemma:lim-in-tf-gen}, \[
    \norm{x-x'}\le\norm{u-u'}=O(D/n)=O(1/n)\;.
\] Since $\RL$ is a Lipschitz function, there exists a constant $C$ such that \[
    \norm{\RL(x)-\RL(x')}\le C\norm{x-x'}=O(1/n)\;.
\]
As $n$ increases, information from non-current time steps will only have a negligible impact on the output of $\RL$. 
\end{proof}

% \section{Discussion on Other Architectures}
% \cam{We supplement theoretical discussions regarding $\texttt{(RNN, RL)}$ and $\texttt{(LRNN, RL)}$.}

% \subsection{Recurrent Neural Network(RNN)}

% \cam{
% For $(\texttt{RNN},\texttt{RL})$, there has been considerable theoretical work demonstrating their completeness on regular languages~\cite{merrill2019sequential,korsky2019computational}. For cases of log precision, based on definitions, we can directly map the recurrent units of RNNs to transition functions in DFAs. Therefore, RNNs do not suffer from the theoretical constraints encountered by Transformers, as validated by the performance of $\texttt{RNN}$ in Figure \ref{fig:regular-hidden} and Figure \ref{fig:regular-learning-curve}.
% }

% \subsection{Linear Recurrent Neural Network(RNN)}
% \cam{

% }

\section{Discussion on Existing POMDP Problems}
\label{sec:discuss-pomdp}
\cam{Using existing POMDP problems as examples, here demonstrates the derivation of POMDPs cast into regular languages.}

\textbf{Passive T-Maze~\cite{tfshine}.} The movement strategy towards the corridor's endpoint is akin to recognizing the regular language $0(01)^*$ with a DFA. If this DFA accepts the string formed by all current histories, then the agent moves upwards; otherwise, it moves downwards.

\textbf{Passive Visual-Match~\cite{hung2019optimizing}.} The complete state space of this environment is large, including player coordinates, coordinates of all fruits, whether fruits are collected, and other information. For convenience, we decouple this POMDP. Consistent with the analysis of Ni et al. (2023), this environment is divided into immediate greedy policies and long-term memory policies. For long-term memory policies, the environment not only needs to recognize the regular language $0(01)^*$ as in Passive T-Maze; but due to the existence of greedy policies, the player needs a strategy to move from any position in the room to the endpoint. The states required for this strategy only involve player coordinates, and judging the current coordinates based on historical information only requires a simple regular language. Considering directly treating the grid of the current room as the states of the DFA, and treating the actions $\{L, R, U, D\}$ as characters, if the player wants to determine the current position $(x, y)$, it is equivalent to recognizing the regular language that accepted by a DFA whose terminal state is $(x,y)$.

\section{Experimental Details}
\label{sec:impl}
\subsection{Task descriptions}
\subsubsection{Pybullet tasks}
\textbf{Ant}. This task is to simulate a hexapod robot resembling an ant. The objective is to develop a control policy that enables the ant to leverage the six legs for specific movements.

\textbf{Walker}. This task is to simulate a bipedal humanoid robot. The goal is to design a control strategy that facilitates stable and efficient walking, mimicking human-like locomotion patterns.

\textbf{HalfCheetah}. This task is to simulate a quadruped robot inspired by the cheetah's anatomy. The aim is to devise a control policy that allows the robot to achieve rapid and agile locomotion.

\textbf{Hopper}. This task is to simulate a single-legged robot, and the objective is to develop a control strategy for jumping locomotion to achieve efficient forward movement.

Task (F) stands for the original task with full observation, while Task (V) and Task (P) stand for that only velocities or positions are observable, respectively. In Figure~\ref{Visualizetasks}, we provide the visualization of each task. 

\begin{figure}[htbp]
    \centering
    \includegraphics[height=0.2\textwidth,width=0.2\textwidth]{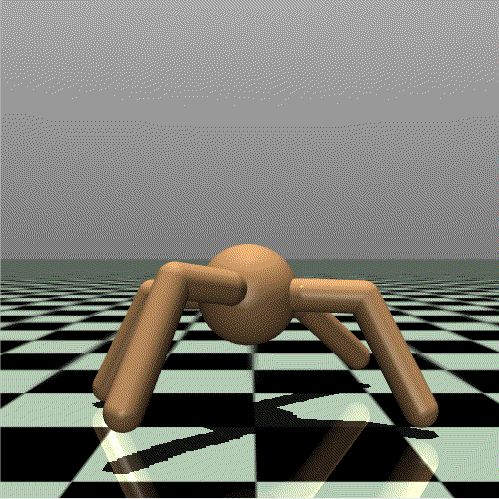}
    \includegraphics[height=0.2\textwidth,width=0.2\textwidth]{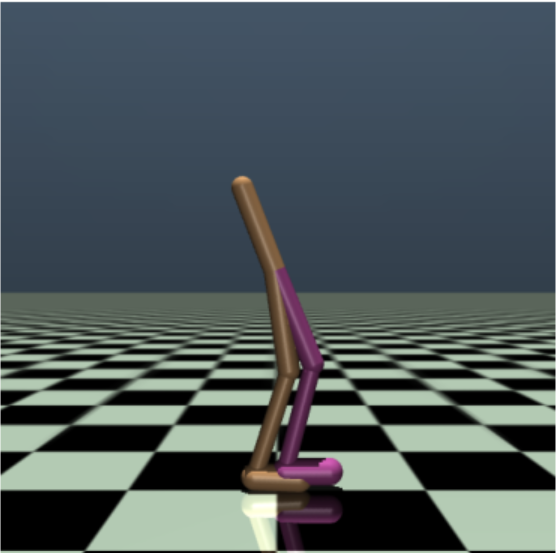}
    \includegraphics[height=0.2\textwidth,width=0.2\textwidth]{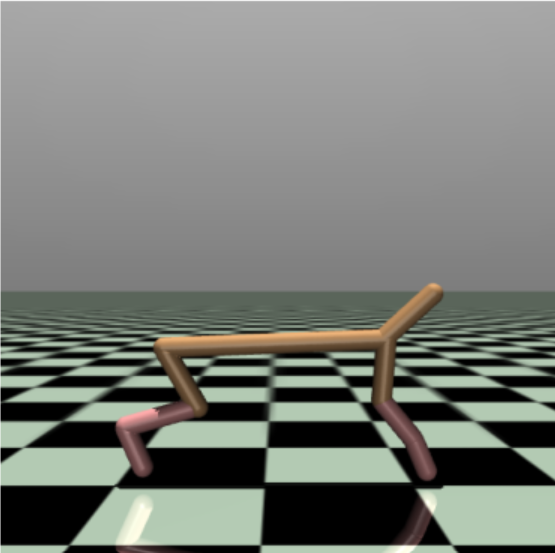}
    \includegraphics[height=0.2\textwidth,width=0.2\textwidth]{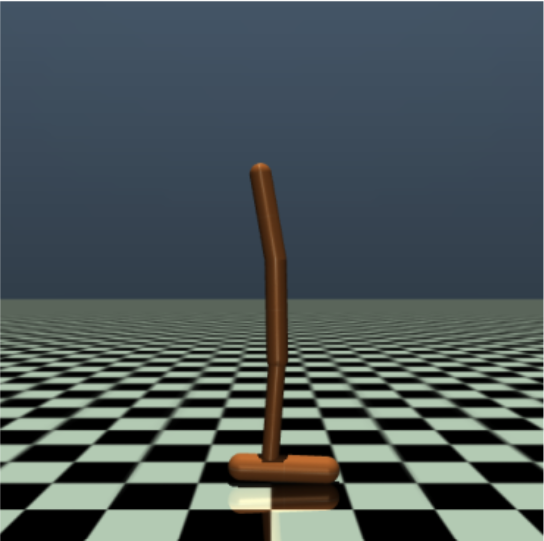}
    \caption{\textbf{Visualizations of Tasks in Pybullet}. From left to right are Ant, Walker, HalfCheetah and Hopper.}
    \label{Visualizetasks}    
\end{figure}
\subsubsection{POMDPs determined by regular languages}
\parity. Given a $01$ sequence, compute whether the number of $1$ is even. The formal expression is $0^*(10^*1)^*0^*$.

\texttt{EVEN PAIRS}. Given a $01$ sequence, compute whether its first and last bit are the same. The formal expression is written as $(0[01]^*0)|(1[01]^*1)$.

\texttt{\sfive}. Given a $01$ sequence, compute whether it belongs to a case of $\NC$ complete regular languages, namely \texttt{S$_5$}, with the formal expression $\nccomplete$.

\subsubsection{Pure long-term memory tasks}
\textbf{Passive T-Maze} (\citet{tfshine}). The environment is a long corridor of length $L$ from the initial state $O$ to $J$ and $J$ is connected with two terminal states $G_1$ and $G_2$. The horizontal positions of $O$ and $J$ are $0$ and $L$, respectively. No information is observable except at states $J,O,G_1,G_2$. At $J,O,G_1,G_2$, the agent could observe its current position $x_t$, and at $O$, the agent could observe a signal $G$, which is uniformly sampled in $\{G_1,G_2\}$. The available actions are to move in $4$ directions: left, right, up and down. The transitions are deterministic, and the agent would not move on hitting the wall. And the goal is to maximize the rewards, which is given by $\sum_{t=0}^L\frac{1(x_t\ge t)-1}{L}+1(o_{L+1}=G)$.

\textbf{Passive Visual Match} (\citet{longtimememory}). The environment is a $7\times 11$ grid-world consisting of $3$ phases while only a $5\times 5$ grid surrounding the agent is observable. In each episode of the task, a randomly assigned color is shown to the agent in the first phase; in the second phase, the agent could pick up apples with immediate rewards; in the final phase, the agent should pick up the assigned color shown among three randomly colored squares.

We show the visualization of each task taken from \citet{tfshine} in Figure~\ref{Visualizelongtasks}. 

\begin{figure}[htbp]
    \centering
    \includegraphics[height=0.13\textwidth]{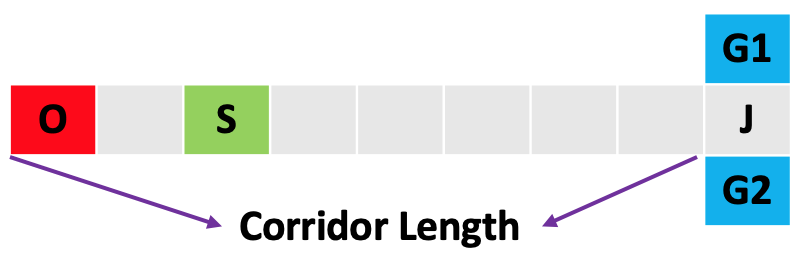}
    \includegraphics[height=0.13\textwidth]{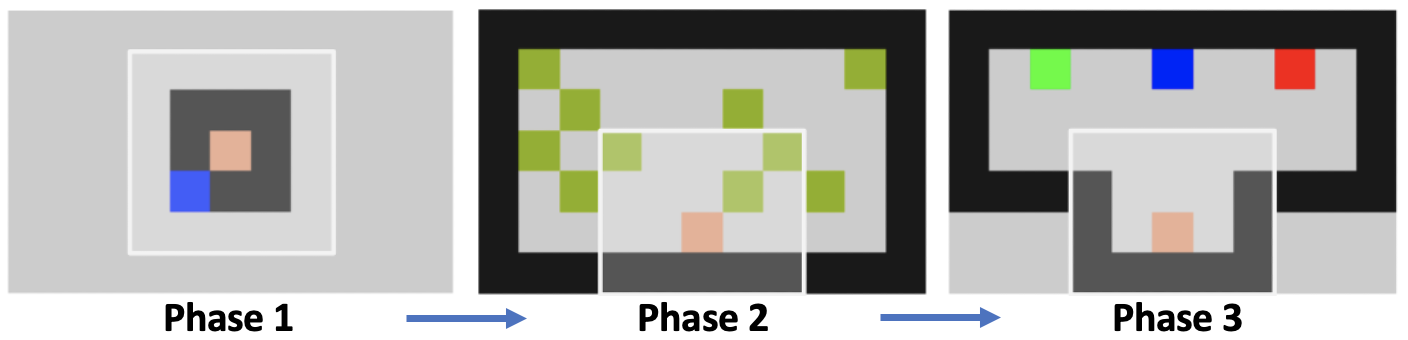}
    
    \caption{\textbf{Visualizations of pure long-term memory Tasks}. The left figure illustrates Passive T-Maze, and the right one indicates Passive Visual Match. Both of them are directly extracted  from \citet{tfshine}.}
    \label{Visualizelongtasks}    
\end{figure}
\subsection{Codebase}
Our code is mainly based on \citet{recurrentbaseline} (\url{https://github.com/twni2016/pomdp-baselines}) and \citet{tfshine} (\url{https://github.com/twni2016/Memory-RL}). The implementation of LRU follows \citet{linearrnn} (\url{https://github.com/yueqirex/LRURec};\url{https://github.com/Gothos/LRU-pytorch}). \cam{
Our official code is released at \url{https://github.com/CTP314/TFPORL}.}

\subsection{Network architecture}
Our network architecture is borrowed from \citet{tfshine}, and for completeness we show the extracted illustrations in Figure~\ref{fig:seq-rl-arch}.
\begin{figure}[htbp]
    \centering
    \includegraphics[height=0.2\textwidth]{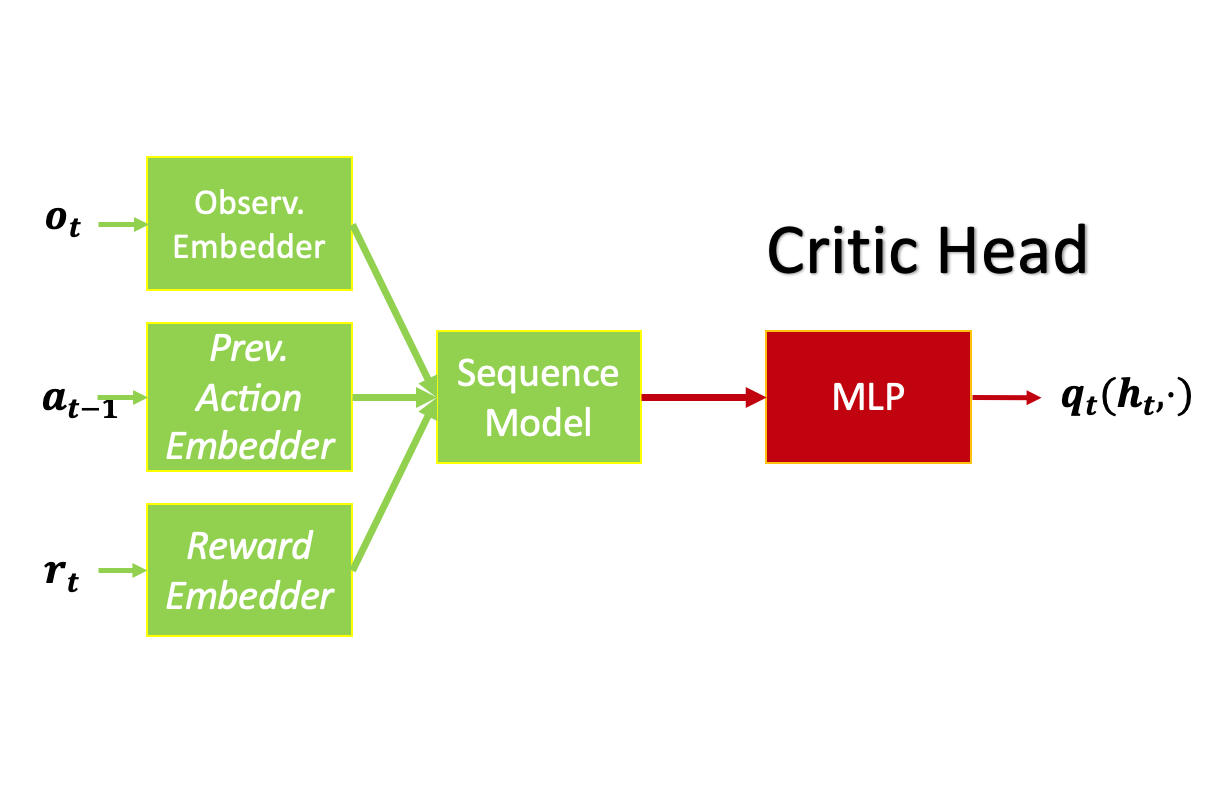}
    \hspace{0.1in}
    \includegraphics[height=0.2\textwidth]{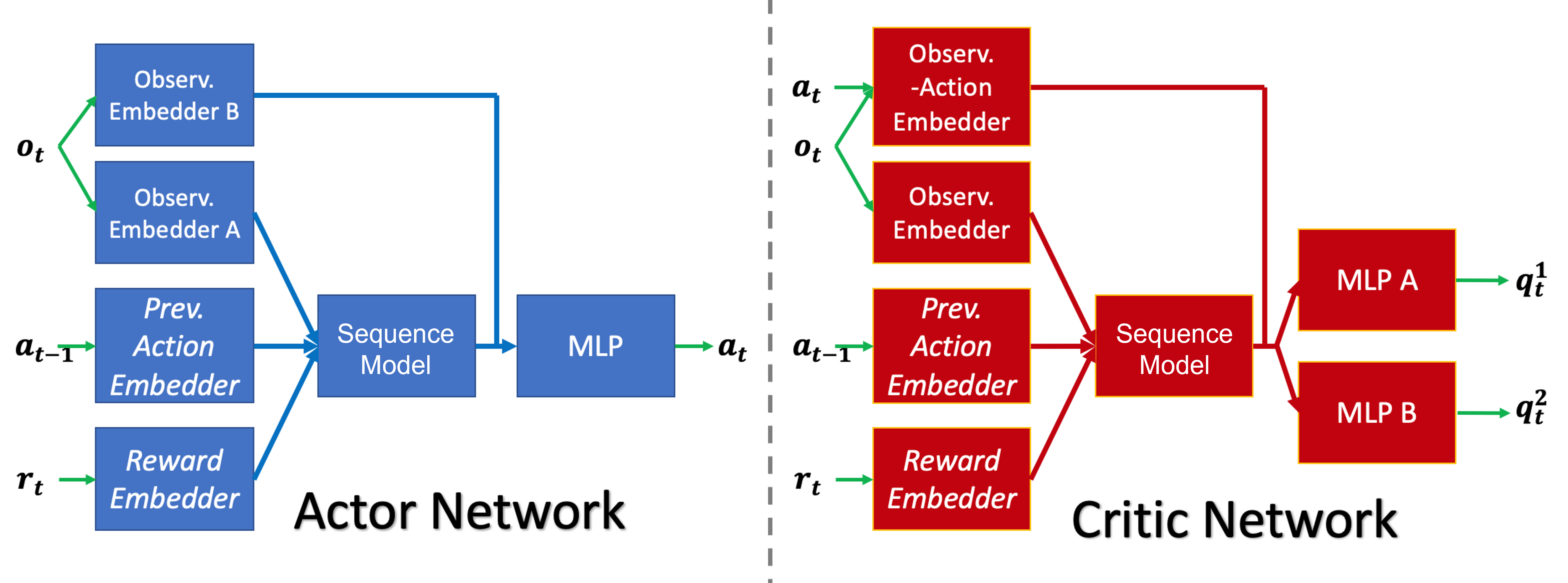}
    \caption{\textbf{The network architecture of $(\SEQ,\RL)$}. The left figure illustrates $(\SEQ,\texttt{DQN})$ and $(\SEQ,\texttt{SACD})$, and the right one indicates $(\SEQ,\texttt{TD3})$. Both of them are directly extracted from \citet{recurrentbaseline,tfshine}.}
    \label{fig:seq-rl-arch}    
\end{figure}
\vspace{-0.1in}
\subsection{Hyperparameters}
We provide the configuration of hyperparameters in Table~\ref{tab:hyperparameters} and Table~\ref{tab:hyperparameters-rl}. Notably, to ensure strong alignment with the existing literature, we select RL algorithms identical to those in \citet{recurrentbaseline, tfshine}, namely \texttt{TD3}~\cite{td3} for PyBullet occlusion tasks, \texttt{DQN}~\cite{ddqn} for Passive T-Maze and \texttt{SACD}~\cite{SAC} for Passive Visual Match. For POMDP derived from Regular Language, due to its similarity with Passive T-Maze (equivalent to recognizing $01$ sequences starting with $0$), we employ \texttt{DQN} as well.

\begin{table*}[htbp]
\centering
\caption{\textbf{Hyperparameters of different POMDP tasks}.}
\label{tab:hyperparameters}
\vspace{0.1in}
\begin{adjustbox}{width=0.8\textwidth}
\begin{tabular}{c|c|cccc}
\toprule
   & Hyperparameter &  Regular Language & PyBullet Occlusion & Passive T-Maze  & Passive Visual Match \\ \midrule
 \multirow{9}{*}{Network} &  LSTM embedding [$h_o,h_a,h_r$]   & 32, 0, 0 & 32, 16, 16 & 32, 16, 0 & 2-layer CNN\\ 
 & LSTM [$n_\text{layer}, h$] & 128, 1 & 128, 1 & 128, 1 &16 \\
  & LRU embedding [$h_o,h_a,h_r$]  & 64, 0, 0 & 32, 16, 16 & 64, 64, 0 & 2-layer CNN \\
  & LRU [$h, n_\text{layer}$] & 64, 2 & 64, 2 & 128, 1 & 128, 1\\
 & GPT embedding [$h_o,h_a,h_r$]  & 64, 0, 0 & 32, 16, 16 &  64, 0, 0 & 2-layer CNN \\
 & GPT [$h, n_\text{heads}, n_\text{layer}$] & 64, 2, 2 & 64, 2, 2 & 128, 1, 1 & 128, 1, 1\\
 & Length & $l_\text{trajectory}$  & 64 & $l_\text{trajectory}$ & $l_\text{trajectory}$ \\
 & RL & \texttt{DQN} & \texttt{TD3} & \texttt{DQN} & \texttt{SACD}  \\
 & MLP & not used  & [256, 256] & [256, 256] & [256, 256] \\ \midrule
% Sequence & Weight initialization & \multicolumn{5}{c}{Orthogonal matrices~\citep{saxe2013exact}} \\ \midrule
\multirow{5}{*}{Train \& Eval} & Environment steps & 0.4M,1M,2M & 1.5M & 1M, 2M, 4M, 4M & 2.4M, 3.75M, 5M  \\ 
& Gradient steps & 40k & 1.5M & 20k, 20k, 16k, 8k & 40k, 30k, 32k \\ 
& Batch size & 64 & 64 & 64 & 64 \\ 
& Evaluation interval  & 100 & 4k & 10 & 50 \\ 
& Evaluation episodes  & 100 & 10 & 10 & 20 \\ 
\bottomrule
\end{tabular}
\end{adjustbox}
\end{table*}

\begin{table*}[htbp]
\centering
\caption{\textbf{Hyperparameters of different RL algorithms}.}
\label{tab:hyperparameters-rl}
\vspace{0.1in}
\begin{adjustbox}{width=0.5\textwidth}
\begin{tabular}{cccc}
        \toprule
        RL algorithm & \texttt{DQN} & \texttt{TD3} & 
        \texttt{SACD} \\
        \midrule
        Hidden dimension & (256, 256) & (256, 256) & (256, 256) \\
        Exploration noise & - & 0.1 & - \\
        Target noise & - & 0.2 & - \\
        Target noise clip & - & 0.5 & - \\
        Discount factor & 0.99 & 0.99 & 0.99\\
        Smoothing coefficient & 0.005 & 0.005 & 0.005\\
        Learning rate & 3e-4 & 3e-4 & 3e-4\\
        Replay buffer size & 1M & 1M & 1M \\
        \bottomrule
\end{tabular}

\end{adjustbox}
\end{table*}

\subsection{Supplementary results}
\label{appendix:supplementary}
In this subsection, we provide additional experimental results for supplementation.  

\textbf{PyBullet occlusion tasks}. In Table~\ref{tab:unormalized-pomdp-task}, we show the original return for each PyBullet occlusion task and sequence model, while the normalized score in Table~\ref{tab:pomdp-task} is calculated as 
\begin{equation}
    \text{Normalize score} = \frac{R-R_{\text{random}}}{R_{\text{expert}}-R_{\text{random}}}\times 100\,,
    \label{eq:normalized_score}
\end{equation}
\note{where $R$ denotes the return of a certain instance, $R_{\text{expert}}$ denotes the return from the expert policy, and $R_{\text{random}}$ denotes the return from a random policy (both provided by \citet{D4RL}).}  Scores used for normalization are provided in Table~\ref{table:normalized score}. The results elucidate that LRU and LSTM demonstrate superior performance over GPT across all eight tasks and LRU distinctly surpasses LSTM in specific tasks, particularly in Ant (P, V), and Walker (V).

\cam{The hyperparameters and code used in our experiments are derived from \citet{recurrentbaseline}, ensuring consistency and fairness with the results presented in Table 7, Figure 9, and Figure 16 of \citet{recurrentbaseline}. We additionally conducted experiments on the Ant environment that show significant deviations from \citet{tfshine} .The training curves in Figure \ref{fig:pomdp-learning-curve-jax} demonstrate that LRU still outperforms LSTM.} 

\begin{table}[htbp]
    \centering
    \caption{\textbf{Scores used for normalization.} Scores of each task are linearly normalized by the corresponding random score and expert score.}
    \vspace{0.1in}
    \begin{adjustbox}{width=0.7\textwidth}
    % \begin{tabular}{ccc}
%         \toprule
%         Task Name & Random Score & Expert Score\\
%         \midrule
%         Ant & 373.71 & 2650.50 \\
%         Walker & 16.52 & 1623.65 \\
%         Cheetah & -1275.77 & 2381.67\\
%         Hopper & 20.06 & 1441.81\\
%         \bottomrule
% \end{tabular}

\begin{tabular}{ccc|ccc}
    \toprule
    % \multicolumn{3}{c|}{Ant and Walker} & \multicolumn{3}{c}{Cheetah and Hopper} \\
    Task Name & Random Score & Expert Score & Task Name & Random Score & Expert Score \\
    \midrule
    Ant & 373.71 & 2650.50 & Cheetah & -1275.77 & 2381.67 \\
    Walker & 16.52 & 1623.65 & Hopper & 20.06 & 1441.81 \\
    \bottomrule
\end{tabular}

    \end{adjustbox}
    \label{table:normalized score}
\end{table}

\begin{table*}[htbp]
\centering
\caption{\textbf{Returns for PyBullet occlusion tasks}. We compare different sequence models LRU, GPT and LSTM. `V' refers to `only velocities observable', and `P' refers to `only positions observable'. \mhl{Blue} highlight indicates the highest score, and \mlhl{orange} highlight indicates the second-highest score.}
\label{tab:unormalized-pomdp-task}
\vspace{0.1in}
\begin{adjustbox}{width=0.6\textwidth}
\begin{tabular}{ccccc}
\toprule
Task & Type & LRU & GPT & LSTM \\ \midrule

\multirow{2}{*}{Ant} & V & \mhl{\po 1052.4}\textcolor{gray}{\mhl{ $\pm$ 465.5\po\po }} & \mlhl{\po\po 587.0}\textcolor{gray}{\mlhl{ $\pm$ 161.6\po\po }} & {\po\po 542.0}\textcolor{gray}{{ $\pm$ 200.0\po\po }} \\
 & P & \mhl{\po 2222.5}\textcolor{gray}{\mhl{ $\pm$ 652.8\po\po }} & \mlhl{\po 1242.5}\textcolor{gray}{\mlhl{ $\pm$ 591.8\po\po }} & {\po\po 504.1}\textcolor{gray}{{ $\pm$ 74.1\po\po\po }} \\
\multirow{2}{*}{Cheetah} & V & \mlhl{\po 2266.2}\textcolor{gray}{\mlhl{ $\pm$ 297.7\po\po }} & {\po 1258.9}\textcolor{gray}{{ $\pm$ 251.5\po\po }} & \mhl{\po 2313.6}\textcolor{gray}{\mhl{ $\pm$ 304.6\po\po }} \\
 & P & \mlhl{\po 2743.1}\textcolor{gray}{\mlhl{ $\pm$ 151.9\po\po }} & {\po 1972.1}\textcolor{gray}{{ $\pm$ 232.2\po\po }} & \mhl{\po 2840.0}\textcolor{gray}{\mhl{ $\pm$ 198.1\po\po }} \\
\multirow{2}{*}{Hopper} & V & \mhl{\po 1357.8}\textcolor{gray}{\mhl{ $\pm$ 332.5\po\po }} & {\po\po 211.4}\textcolor{gray}{{ $\pm$ 7.7\po\po\po\po }} & \mlhl{\po 1193.3}\textcolor{gray}{\mlhl{ $\pm$ 527.1\po\po }} \\
 & P & \mlhl{\po 2122.4}\textcolor{gray}{\mlhl{ $\pm$ 174.7\po\po }} & {\po\po 357.8}\textcolor{gray}{{ $\pm$ 256.4\po\po }} & \mhl{\po 2637.0}\textcolor{gray}{\mhl{ $\pm$ 190.7\po\po }} \\
\multirow{2}{*}{Walker} & V & \mhl{\po 1007.7}\textcolor{gray}{\mhl{ $\pm$ 234.3\po\po }} & \mlhl{\po\po 374.4}\textcolor{gray}{\mlhl{ $\pm$ 95.7\po\po\po }} & {\po\po 212.6}\textcolor{gray}{{ $\pm$ 112.4\po\po }} \\
 & P & \mlhl{\po 1290.3}\textcolor{gray}{\mlhl{ $\pm$ 371.1\po\po }} & {\po\po 812.2}\textcolor{gray}{{ $\pm$ 77.4\po\po\po }} & \mhl{\po 1536.8}\textcolor{gray}{\mhl{ $\pm$ 588.2\po\po }} \\
\bottomrule
\end{tabular}
\end{adjustbox}
\end{table*}

\begin{figure*}[htb]
\vskip 0.2in
\begin{center}
\includegraphics[width=0.3\columnwidth]{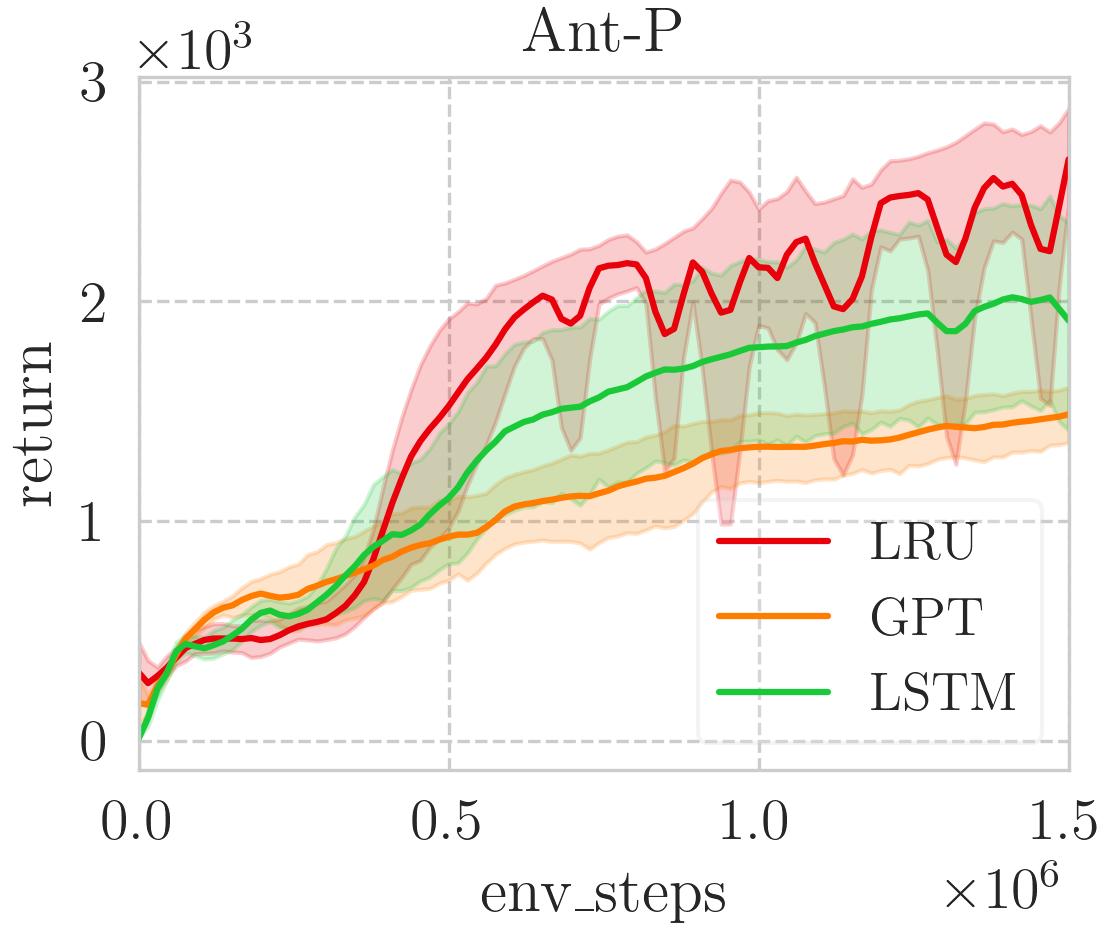}
\hspace{0.3in}
\includegraphics[width=0.3\columnwidth]{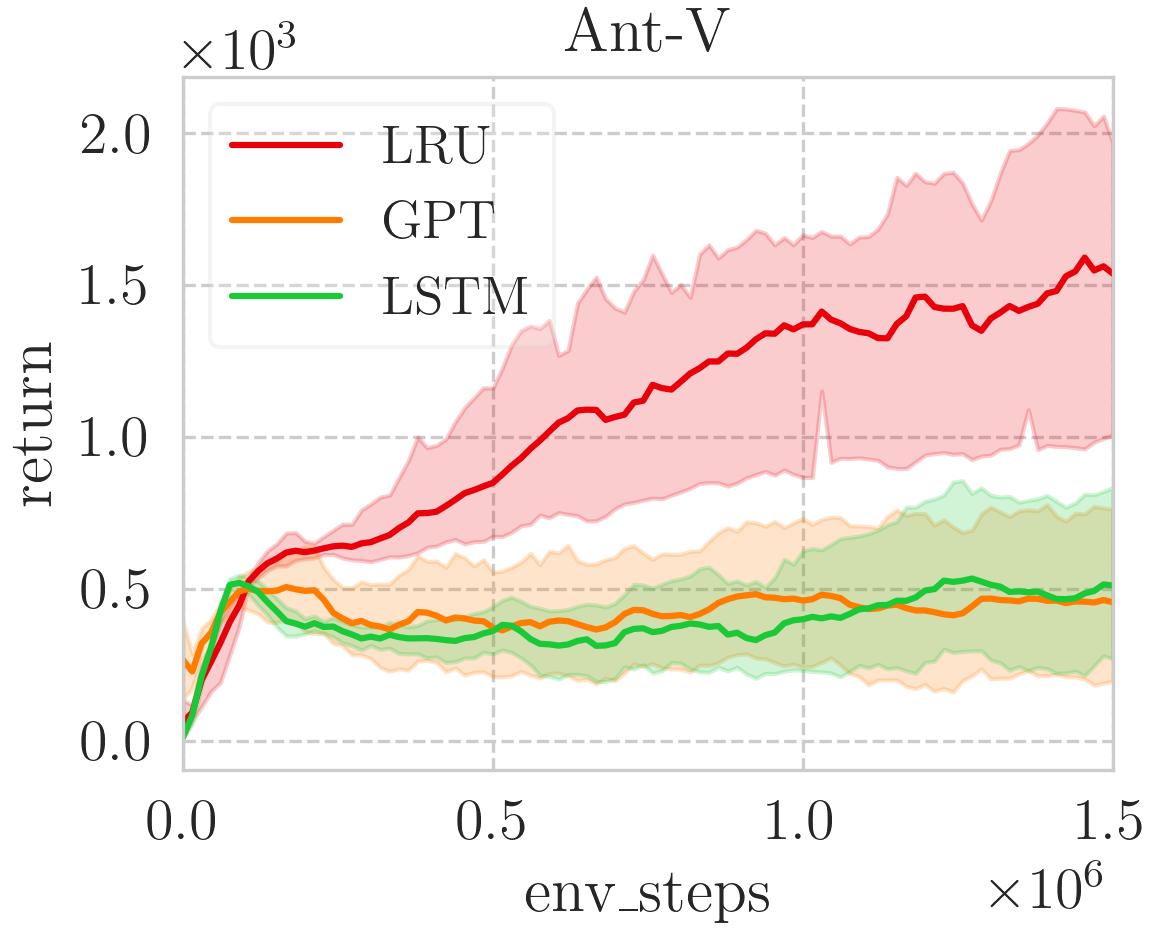}
\end{center}
\caption{\textbf{Learning curves for PyBullet occlusion tasks(Ant) in \citet{tfshine}}. The shaded area indicates $95$\% confidence intervals.}
\label{fig:pomdp-learning-curve-jax}
% \vskip -0.2in
\end{figure*}

\textbf{PyBullet fully observable tasks}. To verify that the performance degradation of GPT compared to the other two models is due to partial observability of the environment, we carry out an ablation study on the partial observability, providing the normalized and original returns for fully observable PyBullet tasks in Table~\ref{tab:pomdp-task-fully-obs}, and learning curves in Figure~\ref{fig:pomdp-learning-curve-fully}. 

\begin{table*}[htbp]
\centering
\caption{\textbf{Returns for PyBullet fully observable tasks}. We compare different sequence models LRU, GPT, and LSTM. `F' refers to `fully observable'. The left table stands for the original score, and the right one indicates the normalized score. \mhl{Blue} highlight indicates the highest score, and \mlhl{orange} highlight indicates the second-highest score. The overall performance of the three models are close, affirming that GPT's inferior performance in POMDP scenarios stems from partial observability rather than other factors.}
\label{tab:pomdp-task-fully-obs}
\vspace{0.1in}
\begin{adjustbox}{width=0.45\textwidth,valign=t}
\begin{tabular}{ccccc}
\toprule
Task & Type & LRU & GPT & LSTM \\ \midrule

Ant & F & \mhl{\po 3218.2}\textcolor{gray}{\mhl{ $\pm$ 42.9\po\po\po }} & \mlhl{\po 2980.8}\textcolor{gray}{\mlhl{ $\pm$ 46.1\po\po\po }} & {\po 2799.2}\textcolor{gray}{{ $\pm$ 122.9\po\po }} \\
Cheetah & F & \mlhl{\po 3101.9}\textcolor{gray}{\mlhl{ $\pm$ 182.3\po\po }} & {\po 2959.6}\textcolor{gray}{{ $\pm$ 424.6\po\po }} & \mhl{\po 3150.4}\textcolor{gray}{\mhl{ $\pm$ 458.5\po\po }} \\
Hopper & F & \mlhl{\po 2613.6}\textcolor{gray}{\mlhl{ $\pm$ 73.2\po\po\po }} & {\po 2533.9}\textcolor{gray}{{ $\pm$ 2.1\po\po\po\po }} & \mhl{\po 2793.7}\textcolor{gray}{\mhl{ $\pm$ 18.9\po\po\po }} \\
Walker & F & {\po 2134.4}\textcolor{gray}{{ $\pm$ 154.9\po\po }} & \mlhl{\po 2152.8}\textcolor{gray}{\mlhl{ $\pm$ 239.4\po\po }} & \mhl{\po 2241.6}\textcolor{gray}{\mhl{ $\pm$ 109.5\po\po }} \\
\bottomrule
\end{tabular}
\end{adjustbox}
\hspace{0.1in}
\begin{adjustbox}{width=0.45\textwidth,valign=t}
\begin{tabular}{cccccc}
\toprule
    Task & Type & LRU & GPT & LSTM \\ Ant & F & \mhl{\po\po 124.9}\textcolor{gray}{\mhl{ $\pm$ 1.9\po\po\po\po }} & \mlhl{\po\po 114.5}\textcolor{gray}{\mlhl{ $\pm$ 2.0\po\po\po\po }} & {\po\po 106.5}\textcolor{gray}{{ $\pm$ 5.4\po\po\po\po }} \\
Cheetah & F & \mlhl{\po\po 119.7}\textcolor{gray}{\mlhl{ $\pm$ 5.0\po\po\po\po }} & {\po\po 115.8}\textcolor{gray}{{ $\pm$ 11.6\po\po\po }} & \mhl{\po\po 121.0}\textcolor{gray}{\mhl{ $\pm$ 12.5\po\po\po }} \\
Hopper & F & \mlhl{\po\po 182.4}\textcolor{gray}{\mlhl{ $\pm$ 5.1\po\po\po\po }} & {\po\po 176.8}\textcolor{gray}{{ $\pm$ 0.1\po\po\po\po }} & \mhl{\po\po 195.1}\textcolor{gray}{\mhl{ $\pm$ 1.3\po\po\po\po }} \\
Walker & F & {\po\po 131.8}\textcolor{gray}{{ $\pm$ 9.6\po\po\po\po }} & \mlhl{\po\po 132.9}\textcolor{gray}{\mlhl{ $\pm$ 14.9\po\po\po }} & \mhl{\po\po 138.5}\textcolor{gray}{\mhl{ $\pm$ 6.8\po\po\po\po }} \\
\midrule
\multicolumn{2}{c}{\textbf{Average}} & {\po\po139.7\po\po} & {\po\po135.0\po\po} & \mhl{\po\po140.3\po\po} \\
\bottomrule
\end{tabular}
\end{adjustbox}
\end{table*}

\begin{figure*}[htbp]
\vskip 0.2in
\begin{center}
\centerline{\includegraphics[width=\textwidth]{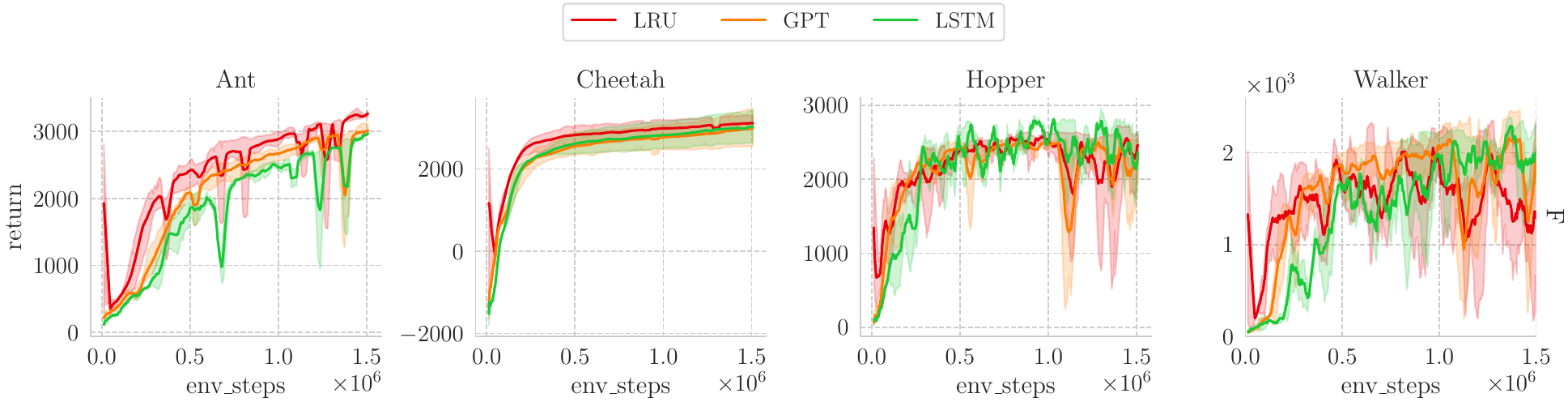}}
\caption{\textbf{Learning curves for PyBullet fully observable tasks}. The shaded area indicates $95$\% confidence intervals.}
\label{fig:pomdp-learning-curve-fully}
\end{center}
\vskip -0.2in
\end{figure*}

\begin{figure*}[htbp]
\vskip 0.2in
\begin{center}
\centerline{\includegraphics[width=\textwidth]{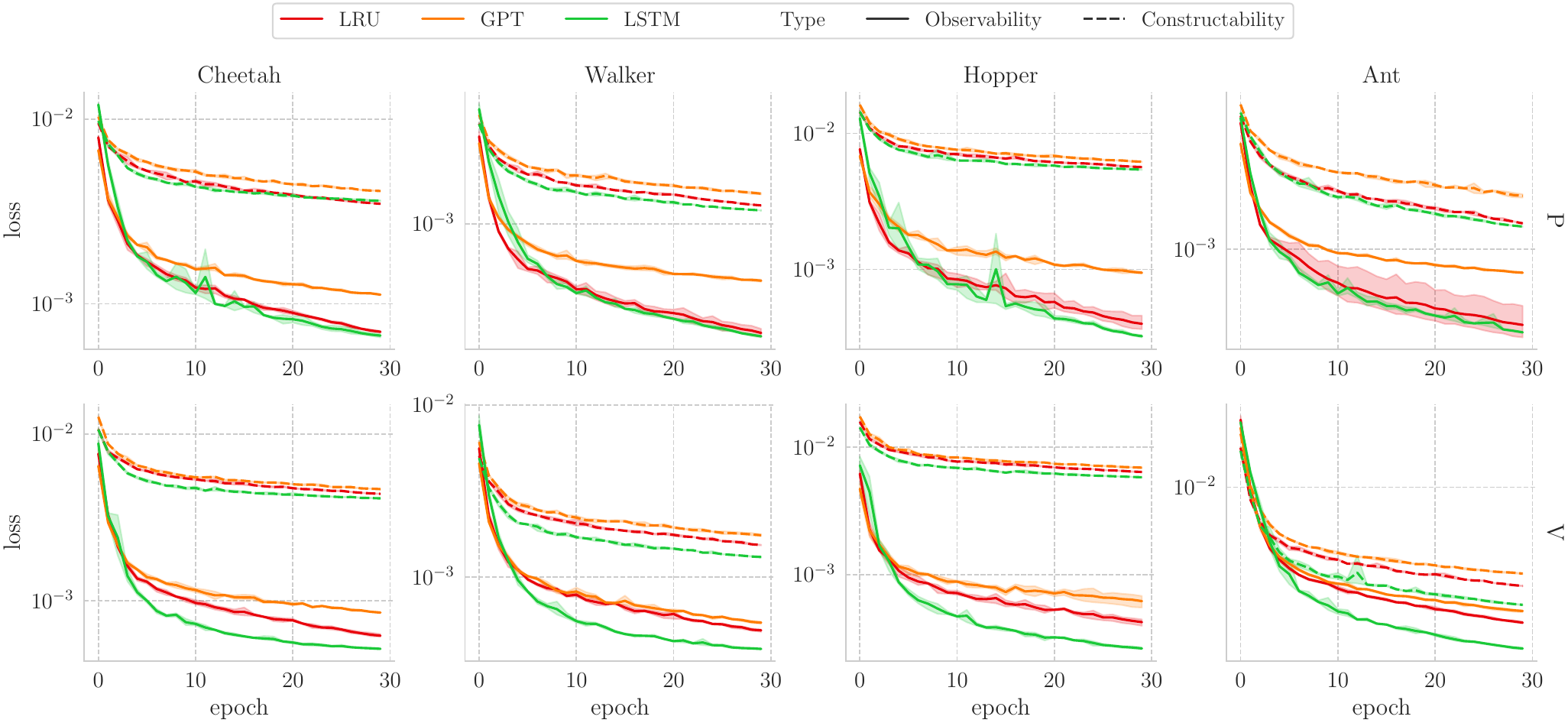}}
\caption{\textbf{Learning curves for \ssm tasks}. We present the curves of average MSE loss over $3$ seeds. The shaded area indicates $95$\% confidence intervals.}
\label{fig:pomdp-learning-curve-ssm}
\end{center}
\vskip -0.2in
\end{figure*}

\textbf{PyBullet \ssm tasks}. We provide the MSE loss curve throughout training for PyBullet \ssm tasks in Figure~\ref{fig:pomdp-learning-curve-ssm}, while in Figure~\ref{fig:ssm}, we present the average MSE ratios over the 8 tasks. The MSE ratio is computed as 
\begin{equation}
\text{MSE ratio} = {l}/{l_{\min}}\;,
\label{eq:mse_ratio}
\end{equation}
where $l$ denotes the final MSE loss of the training curve, and $l_{\min}$ denotes the minimal $l$ over the 3 models.

% \textbf{Mujoco Frame Dropping Task with Offline RL} \cam{Add it!}

\textbf{Length Extrapolation in Regular Language Tasks}. \cam{Figure \ref{fig:length-extra} shows the results for regular language tasks with training lengths less than or equal to $n=25$, tested at $n/2$ and $n+i$ for $i\in\{1,2,3,4,8,16,32\}$. Due to RNN's completeness over regular languages, RNN exhibits the best generalization ability. On $\texttt{SYM(5)}$ and $\texttt{PARITY}$, LRU significantly outperforms GPT, which aligns with the discussion in Sec 5, mitigating the theoretical constraints of GPT. On $\texttt{EVEN PAIRS}$, which naturally conforms to attention's inductive bias, GPT also demonstrates better generalization ability than LRU, but it fails at longer lengths (all three languages fall within the scope of Theorem \ref{thm:lim-in-tfrl-gen}). This result is consistent with the visualization results in Figure 3, where LRU shows good recovery effects for hidden states on $\texttt{SYM(5)}$ and $\texttt{PARITY}$, while GPT's hidden states exhibit significant discrepancies from the true distribution.}

\begin{figure*}[htb]
\vskip 0.2in
\begin{center}
\centerline{\includegraphics[width=0.65\textwidth]{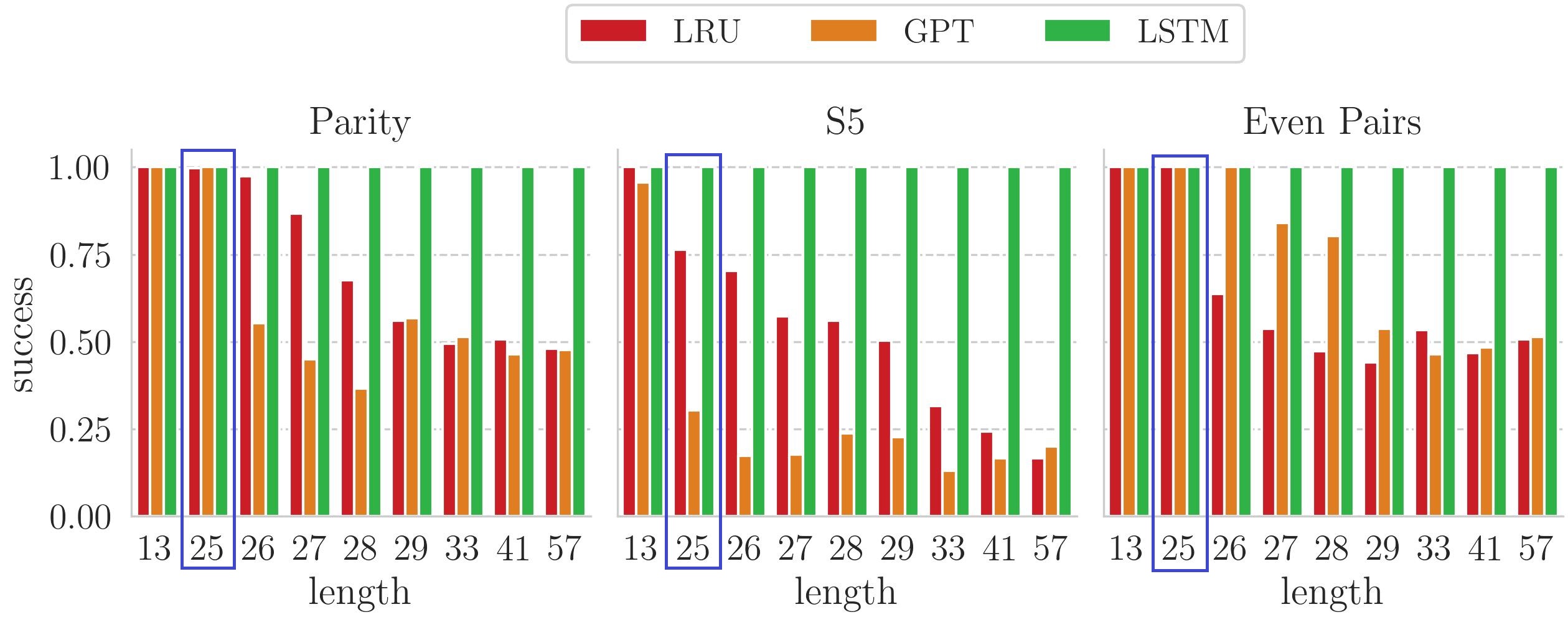}}
\caption{\textbf{Length Extrapolation in Regular Language Tasks}. Mean of 5 seeds. Training length with the longest length enclosed in a blue box.}
\label{fig:length-extra}
\end{center}
\vskip -0.2in
\end{figure*}

\textbf{GPT Scale in Regular Language Tasks}. \cam{Figure \ref{fig:gpt-scale} shows learning curves of different scales on the $\texttt{SYM(5)}$($n=25$) task with Transformer. The results show that scaling up GPT does not yield significant effects, consistent with the statement of Theorem~\ref{thm:lim-in-tfrl-gen}. Transformer is only equivalent to a computation circuit of width $O(\text{poly}(n))$ and depth $O(1)$. Changing the parameter count merely increases the width or changes the constant width, while not significantly enhancing the Transformer's ability to solve $\mathsf{NC^1}$ complete problem.}

\begin{figure*}[htb]
\vskip 0.2in
\begin{center}
\centerline{\includegraphics[width=0.8\textwidth]{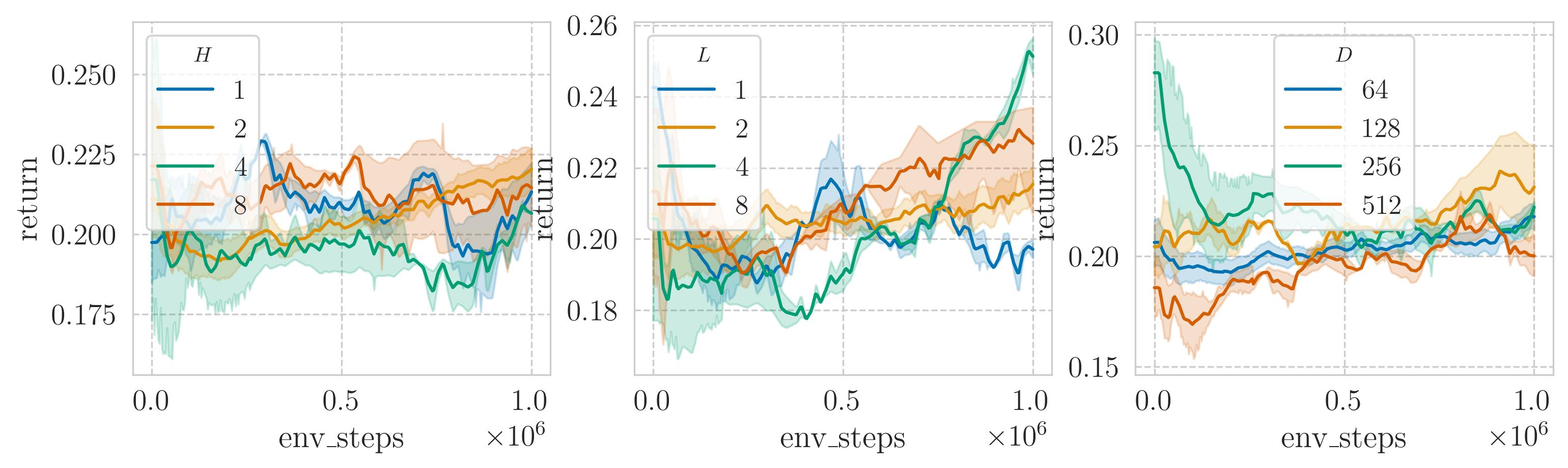}}
\caption{\textbf{GPT Scale in $\texttt{SYM(5)}$}. Mean of 5 seeds. $H,L,D$ denote hidden size, the number of layers and the depth, respectively.}
\label{fig:gpt-scale}
\end{center}
\vskip -0.2in
\end{figure*}

\textbf{Regular language tasks and long-term memory tasks}. The learning curves for regular language tasks are shown in Figure~\ref{fig:regular-learning-curve}: in \texttt{EVEN PAIRS}, all the sequence models could solve this task even in the hardest case; in \parity, LSTM shows significant advantages over LRU and GPT; in \sfive, GPT performs the worst, and is not capable of solving the medium case. 

In this table \ref{table:clustering-metric}, the silhouette scores~\cite{rousseeuw1987silhouettes} for each task are shown for different models to demonstrate the quality of hidden state representations.

\begin{itemize}
    \item In $\texttt{PARITY}$ and $\texttt{SYM(5)}$ tasks, both LRU and LSTM significantly outperform GPT, consistent with the performance shown in Figure~\ref{fig:regular-hidden}. In $\texttt{PARITY}$, both LRU and LSTM implicitly learn the true state within the $(\texttt{SEQ},\texttt{RL})$ framework, hence their scores are relatively close. However, in the more complex DFA structure of $\texttt{SYM(5)}$, there remains a gap between LRU and LSTM, consistent with the training situation depicted in Figure~\ref{fig:regular-hidden}.
    \item In $\texttt{EVEN PARIS}$, LSTM performs noticeably better than the other two models. Both LRU and GPT struggle to learn the true state. For LRU, it faces difficulty in capturing information about the first token from history. For GPT, although its state reconstruction in intermediate moves is poor, it manages to retain contextual information throughout, allowing it to directly extract information about the first token at termination states. This aligns with the visualization results where termination states match previous states, and each termination state exhibits significant separation, as depicted.
\end{itemize}

The learning curves for pure long-term memory tasks are shown in Figure~\ref{fig:memory-learning-curve}: LRU and GPT far outweigh LSTM in these two tasks, demonstrating their long-term memory abilities which LSTM prominently lacks. Although in the Passive T-Maze, GPT could solve the hardest case while LRU and LSTM could not, LRU exhibits strong sample efficiency and memory capacity in the more complicated task, Passive Visual Match, beating GPT and LSTM remarkably.

\begin{center} 
\begin{figure}[ht]
\vskip 0.2in
\subfloat[\texttt{EVEN PAIRS}]{
    \includegraphics[width=\columnwidth]{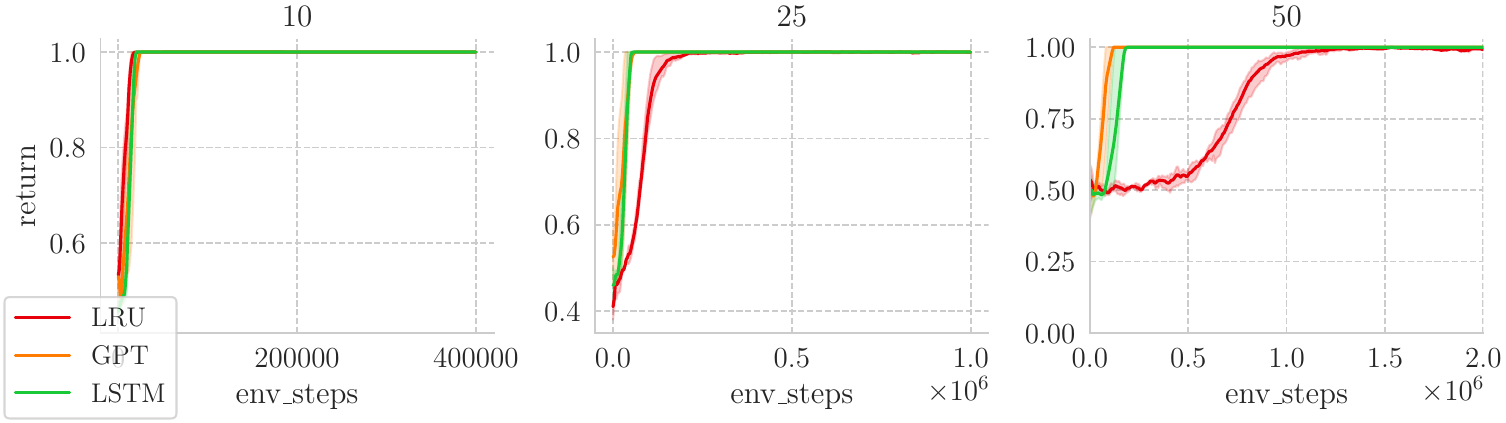}
}
\newline
\subfloat[\parity]{
    \includegraphics[width=\columnwidth]{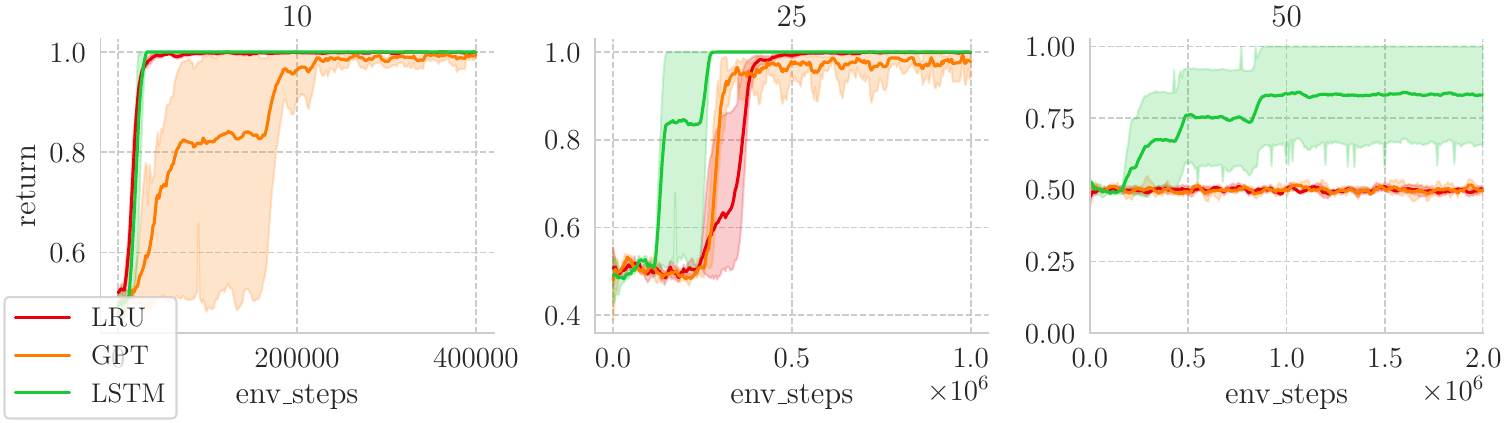}
}
\newline
% \hspace{-0.01in}
\subfloat[\sfive]{
    \includegraphics[width=\columnwidth]{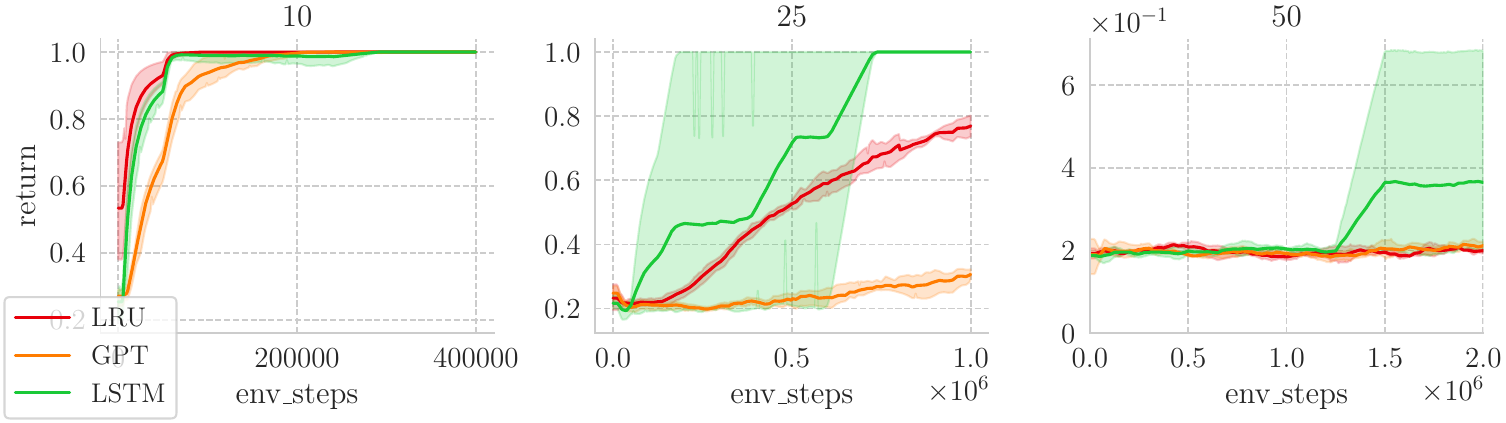}
}
\caption{\textbf{Learning curves for regular language tasks}. The shaded area indicates $95$\% confidence intervals. The upper number stands for the memory length.}
\label{fig:regular-learning-curve}
\vskip -0.2in
\end{figure}
\end{center}

\begin{table}[htbp]
    \centering
    \caption{\textbf{Scores for hidden state representation in regular language tasks} The silhouette score ranges from -1 to 1, where a score closer to 1 indicates that the samples are well-clustered and far from neighboring clusters, signifying a good separation between different states and good cohesion within the same state.}
    \vspace{0.1in}
    \begin{adjustbox}{width=0.3\textwidth}
    \begin{tabular}{cccc}
    \toprule
    % \multicolumn{3}{c|}{Ant and Walker} & \multicolumn{3}{c}{Cheetah and Hopper} \\
    Task & LRU & GPT & LSTM \\
    \midrule
    $\texttt{PARITY}$ & 0.12 & -0.03 & 0.1 \\
    $\texttt{SYM(5)}$ & 0.05 & -0.13 & 0.12 \\
    $\texttt{ENV PARIS}$ & 0.02 & -0.02 & 0.2 \\
    \bottomrule
\end{tabular}

% | Task       |   LRU |   GPT |   LSTM |
% |:-----------|------:|------:|-------:|
% | $\texttt{PARITY}$     |  0.12 | -0.03 |   0.1  |
% | $\texttt{SYM(5)}$         |  0.05 | -0.13 |   0.12 |
% | $\texttt{EVEN PARIS}$ |  0.02 | -0.02 |   0.2  |
    \end{adjustbox}
    \label{table:clustering-metric}
\end{table}

\begin{figure}[ht]
\vskip 0.2in
\begin{center}
\subfloat[Passive T-Maze (return)]{
    \includegraphics[width=\columnwidth]{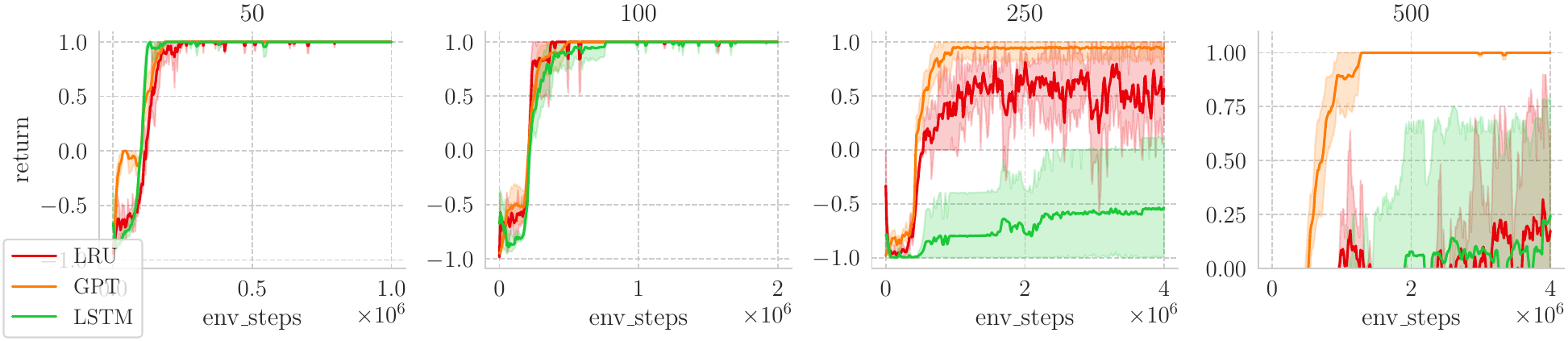}
}
\newline
\subfloat[Passive Visual Match (return)]{
    \includegraphics[width=\columnwidth]{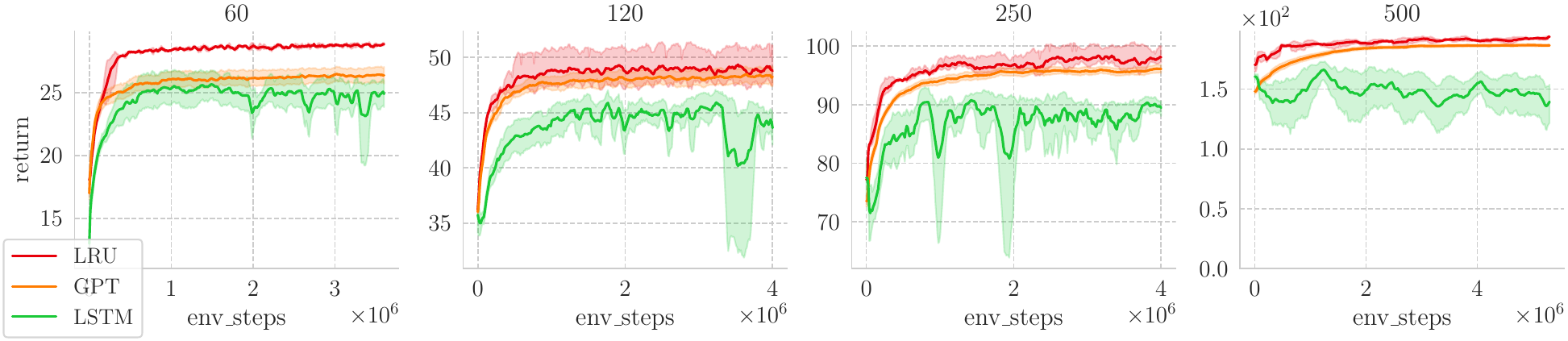}
}
\newline
\subfloat[Passive Visual Match (success)]{
    \includegraphics[width=\columnwidth]{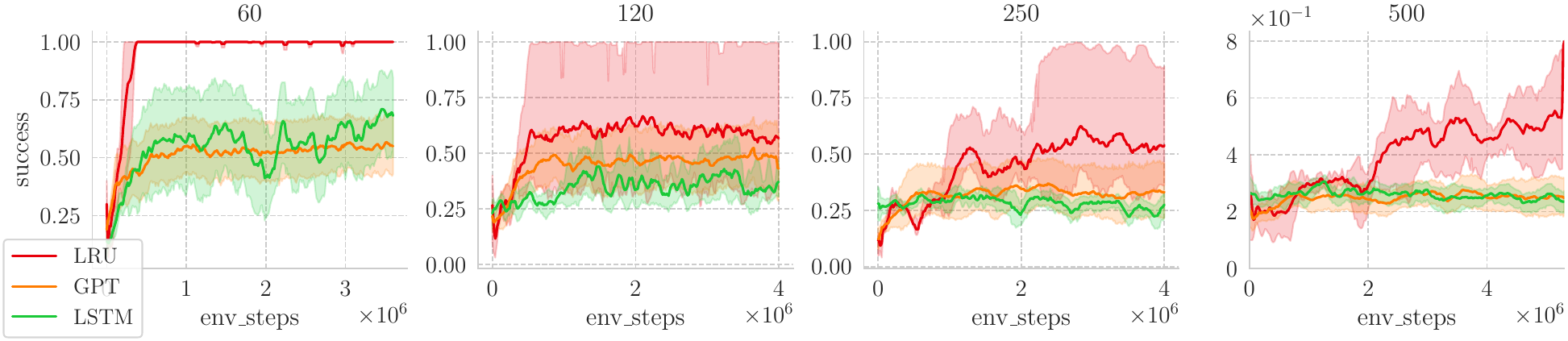}
}
\caption{\textbf{Learning curves for pure long-term memory tasks}. Mean of $3$ seeds. The shaded area indicates $95$\% confidence intervals. The upper number stands for memory length.}
\label{fig:memory-learning-curve}
\end{center}
\vskip -0.2in
\end{figure}

\textbf{Comparison with some published Transformer in RL}. We compared another practical POMDP environment: control tasks with random frame dropping~\cite{defog}, and the popular DT~\cite{dt} used in current Offline RL approaches. When applied to the processed dataset, DT failed to learn a control policy in the frame-dropping situation and required additional information to mitigate the partially observable condition. However, when replacing Transformer with RNN or Linear RNN, DRNN or DLRNN exhibited the ability to handle such problems without additional conditions. Figure~\ref{fig:defog}

\begin{figure*}[htb]
\vskip 0.2in
\begin{center}
\centerline{\includegraphics[width=0.8\textwidth]{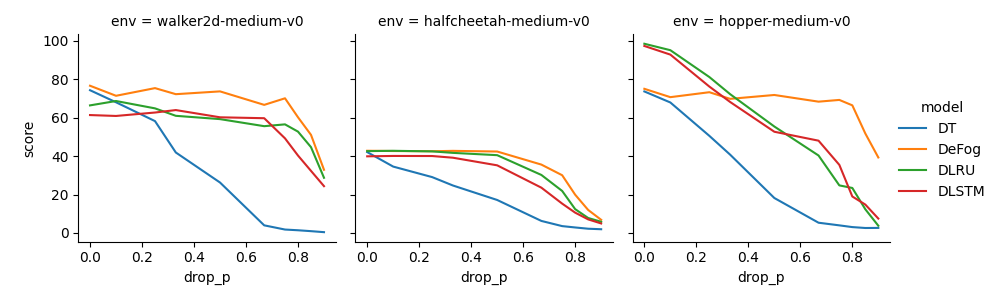}}
\caption{\textbf{Performance on continuous control tasks. with random frame dropping.} Mean of 5 seeds. We selected these medium datasets as the training set. The x-axis represents the probability of dropped frames, and the y-axis represents the D4RL normalized score~\cite{D4RL}.}
\label{fig:defog}
\end{center}
\vskip -0.2in
\end{figure*}
\end{document}